\documentclass{article} % For LaTeX2e
\usepackage{iclr2027_conference,times}

\usepackage{graphicx}

\usepackage{amsmath,amsfonts,bm}

\def\eqref#1{equation~\ref{#1}}
\def\1{\bm{1}}

\DeclareMathAlphabet{\mathsfit}{\encodingdefault}{\sfdefault}{m}{sl}
\SetMathAlphabet{\mathsfit}{bold}{\encodingdefault}{\sfdefault}{bx}{n}

\usepackage{hyperref}
\hypersetup{hidelinks}
\usepackage{url}
\usepackage{booktabs}
\usepackage{amsthm}
\usepackage{amsfonts}
\usepackage{mathrsfs}
\usepackage{multirow}
\newtheorem{proposition}{Proposition}
\usepackage{amsmath}
\usepackage{float}
\usepackage{subcaption}
\usepackage{wrapfig}
\usepackage{bm}
\newcommand{\best}[1]{\bm{#1}}
\newcommand{\second}[1]{\underline{#1}}
\DeclareCaptionJustification{fulljustify}{%
  \leftskip=0pt
  \rightskip=0pt
  \parfillskip=0pt
}

\title{STOCK-JEPA: Prior-Anchored Latent Revision Representation Learning in Equity Markets}

\author{%
Yizhi Luo\textsuperscript{1}, Jiahe Yi\textsuperscript{1},
Jianhui Zhang\textsuperscript{2,3}, Shuo Sun\textsuperscript{1}\\[4pt]
{\normalfont\small \textsuperscript{1}The Hong Kong University of Science and Technology (hkust-gz.edu.cn)}\\
{\normalfont\small \textsuperscript{2}Shanghai Jiaotong University (sjtu.edu.cn)}\\
{\normalfont\small \textsuperscript{3}Beijing Zhongguancun Academy (bza.edu.cn)}
}

\iclrfinalcopy % Uncomment for camera-ready version, but NOT for submission.

\begin{document}

\maketitle
\lhead{Preprint}

% Effective representations help characterize the structure and dynamics of equity markets from high-dimensional financial data with low signal-to-noise ratio. Flexible deep models can capture complex patterns but may fit sample noise and lack explicit economic structure. Financial models provide interpretable references, yet can leave predictable information uncaptured. To combine these strengths

\begin{abstract}
Learning effective representations helps characterize the structure and dynamics of equity markets from financial data with a low signal-to-noise ratio. Black-box deep models can capture complex patterns but may overfit sample noise and lack explicit economic structure. Meanwhile, classic linear financial models provide interpretable references, but their oversimplified assumptions leave non-linear signals uncaptured. To combine the strengths of these two directions, we propose \textsc{Stock-JEPA}, a joint-embedding predictive framework that learns predictable incremental revisions relative to a point-in-time financial prior. First, we leverage a low-complexity financial model to produce fixed statistics summarizing multi-horizon return and risk. A prior projector then maps these statistics into the target encoder's latent space as an anchor. Second, we design a context-conditioned revision predictor to estimate the future representation's predictable displacement from the anchor. Separate losses update the two branches: the anchor learns from prior statistics, while the revision captures additional predictable information from historical context. Third, we freeze all representation modules and train a downstream readout, evaluating its forecasts through cross-sectional ranking and portfolio performance. Theoretically, we prove that optimal revision reduces the prior anchor's expected squared error for the same future representation by exactly $\mathbb{E}[\|\boldsymbol{\Delta}\|_2^2]$. This non-negative gain is the expected squared magnitude of the additional signal predictable from historical context. Experimentally, \textsc{Stock-JEPA} outperforms 13 strong baselines across large-scale China and U.S. equity universes on 5 key evaluation metrics. Ablation studies and representation analysis further demonstrate the value of the learned revisions for representation learning in equity markets.
\end{abstract}

\section{Introduction}
\label{sec:introduction}
Financial representation learning aims to characterize market structure and return-risk dynamics from noisy financial data, providing reusable representations for relevant tasks such as return forecasting and portfolio management. However, the low signal-to-noise nature of financial data leads to poor predictive signals, and future outcomes contain insider information unavailable to the public \citep{kyle1985continuous}. Advanced deep learning models have the ability to capture complex non-linear patterns, but the learned relationships do not automatically admit a clear economic interpretation \citep{gu2020empirical,giglio2022factor}. End-to-end training using realized future outcomes as targets may also overfit sample-specific noise, weakening out-of-sample generalization \citep{wang2021coupling}. These limitations motivate enhancing deep models with traditional financial models, which provide structured and interpretable priors.

Traditional financial factor models describe each stock's return as a weighted combination of a few shared factors, such as broad market movements, plus an asset-specific component \citep{ross1976arbitrage}. These weights, called factor exposures, indicate how strongly the stock responds to each factor. Conditional factor models allow the weights to change with market conditions and firm characteristics available at prediction time \citep{ferson1999conditioning,kelly2019characteristics}. These models provide a compact, financially grounded prior, but they are linear in form and can leave non-linear interactions unexplained. Residual learning offers a natural extension: a financial model supplies a baseline, while a deep model uses historical context to learn additional predictive structure \citep{avellaneda2010statistical,guijarro2025statistical}. However, the observed residual can mix omitted predictive signals with unpredictable noise (example in Figure~\ref{fig:synthetic_omitted_span}). If we instead predict a learned representation of the future state, the target can summarize return and risk patterns, giving the model flexibility to emphasize predictable structure over the details of realized outcomes. Joint-Embedding Predictive Architectures (JEPAs) support this approach by predicting encoded targets without reconstructing complete observations \citep{lecun2022path,assran2023selfsupervised}. To adapt JEPAs to equity markets, one major challenge is that the vanilla generic objectives do not explicitly distinguish what a financial prior can predict from what historical context can add.

This motivates a simple question: can we start from a financial prior and directly learn predictable revisions to future representations? We propose \textsc{Stock-JEPA}, a joint-embedding predictive framework that learns incremental latent revisions relative to a point-in-time financial prior. A low-complexity financial model produces fixed statistics of multi-horizon return and risk, which a prior projector maps into the target encoder's latent space to form an anchor. A context-conditioned revision predictor then uses historical observations to estimate the future representation's predictable displacement from this anchor. To prevent the two branches from compensating for each other, we introduce separate objectives: the prior branch predicts the future representation from prior statistics, while the context branch learns the remaining predictable component. We freeze all representation modules and train only a downstream readout to assess their predictive and portfolio utility. 
When predicting the same future representation, adding historical context can only maintain or reduce the minimum achievable mean squared error relative to the prior alone. The optimal revision captures this additional predictability, reducing the error by exactly the non-negative quantity
$\mathbb{E}[\|\boldsymbol{\Delta}\|_2^2]$.

We conduct experiments on full-size China and U.S. equity universes with consideration of realistic settings (e.g., chronological splits, transaction cost and long-short constraints). Among 13 baselines, \textsc{Stock-JEPA} achieves the best performance across 4 out of 5 evaluation metrics in each market, improving RankIC by 14.0\% and 15.9\% over the second-best method. Ablation studies support the value of the prior-anchored pipeline, the incremental contribution of revisions, and the portfolio benefits of combining separate objectives with an EMA anchor. Structural and temporal diagnostics reveal stock-level variation and broad revision adjustments around market events. Controlled simulation further supports our formulation: the anchor captures predictable structure from the prior, while the revision learns complementary signals from historical context, even with noisy future targets.

Our contributions are fourfold:
\begin{itemize}
    \item We formulate financial representation learning as learning predictable revisions beyond a structured prior, separating what history can help predict from what remains unpredictable.
    \item We prove that, to forecast the same future representation, adding historical context reduces the best possible squared prediction error by exactly $\mathbb{E}[\|\boldsymbol{\Delta}\|_2^2]$ when history helps predict.
    \item We introduce \textsc{Stock-JEPA}, a joint-embedding framework that combines structured financial priors with deep predictive learning through distinct anchor and revision objectives.
    \item Frozen \textsc{Stock-JEPA} representations achieve the highest mean RankIC and net Sharpe among evaluated methods on large-scale China and U.S. equity universes, with ablations supporting the revisions' incremental value and controlled simulation showing that the
    anchor and revision learn complementary return signals from noisy data.
\end{itemize}

\section{Related Work}
\label{sec:related_work}

\subsection{Financial Models and Residual Learning}
\label{sec:related_financial_models}

Conditional factor models provide a structured prior for returns. IPCA uses firm characteristics to model time-varying exposures to latent common factors \citep{kelly2019characteristics}. Residual-based strategies examine temporal signals left outside this structure. Avellaneda and Lee model idiosyncratic returns as mean-reverting processes \citep{avellaneda2010statistical}, while Deep Learning Statistical Arbitrage extracts trading signals from the histories of factor-residual portfolios \citep{guijarro2025statistical}. These strategies derive trading signals from factor-adjusted returns, but operate in return space and do not define a prior for predicting future-state representations. \textsc{Stock-JEPA} instead uses a financial anchor to define predictable revisions in representation space.

\subsection{Deep Financial Factor Models}
\label{sec:related_deep_factors}

Deep factor models combine financial structure with neural flexibility. Conditional autoencoders learn non-linear mappings from firm characteristics to factor exposures \citep{gu2021autoencoder}, while FactorVAE models latent factor distributions through variational inference \citep{duan2022factorvae}. Other methods explicitly combine financial priors with learned factors. FactorGCL extracts complementary factors from residual stock embeddings using hypergraphs and temporal contrastive learning \citep{duan2025factorgcl}, while PRISM-VQ combines expert factors with vector-quantized latent factors for stock ranking \citep{kim2026prismvq}. These methods demonstrate how financial structure can guide deep learning, but organize the learned representations as factors or exposures for return prediction. \textsc{Stock-JEPA} instead keeps a financial prior as an anchor, without requiring the learned future-state revisions to admit a factor interpretation.

\subsection{Predictive Representation Learning}
\label{sec:related_predictive}

JEPAs predict target embeddings from observable context without reconstructing raw observations \citep{assran2023selfsupervised}. TS-JEPA adapts this principle to masked time-series patches \citep{ennadir2024joint}. Financial applications learn market-level states from equity cross-sections \citep{mahns2026joint} or temporal representations of individual stocks \citep{wang2026finjepa}. Phys-JEPA partitions latent states into physical and residual components, using observed physical variables to constrain the physical component and its temporal changes \citep{nie2026physjepa}. Predicting latent targets or partitioning latent states does not itself distinguish what a prior can predict from what richer context adds. \textsc{Stock-JEPA} addresses this gap through a financial anchor and separate objectives for prediction from the prior and revision from historical context.

\section{Problem Formulation}
\label{sec:formulation}

\paragraph{Information structure.}
When predicting the same target, more historical information can only maintain or reduce the minimum achievable mean squared error \citep{banerjee2005optimality}. We study prior-relative representation learning through this standard conditional-expectation principle. At forecast origin $t$, the learner observes the fixed financial prior statistics $\mathbf{P}_{i,t}$ and historical context $\mathbf{X}_{i,t}^{\mathrm{ctx}}$; the future window $\mathbf{X}_{i,t}^{\mathrm{tgt}}$ is used only as a representation-learning target. All variables share a probability space $(\Omega,\mathcal{F},\mathbb{P})$. For population analysis, the target encoder is fixed and
$\mathbf{Z}_{i,t}^{\mathrm{tgt}}\in
L^2(\Omega,\mathcal{F},\mathbb{P};\mathbb{R}^{d_z})$.
Suppressing asset and date subscripts, define
\begin{equation}
\label{eq:formulation_information_sets}
    \mathcal{G}:=\sigma(\mathbf{P}),
    \qquad
    \mathcal{H}:=\sigma(\mathbf{P},\mathbf{X}^{\mathrm{ctx}}),
    \qquad
    \mathcal{G}\subseteq\mathcal{H}.
\end{equation}
\paragraph{Prior-relative predictability.} The prior anchor and predictable revision are
\begin{equation}
\label{eq:formulation_revision}
\begin{aligned}
    \mathbf{Z}^{\mathrm{anc}}
    &:=
    \mathbb{E}[\mathbf{Z}^{\mathrm{tgt}}\mid\mathcal{G}],
    \\
    \boldsymbol{\Delta}
    &:=
    \mathbb{E}[\mathbf{Z}^{\mathrm{tgt}}\mid\mathcal{H}]
    -
    \mathbf{Z}^{\mathrm{anc}}.
\end{aligned}
\end{equation}
The anchor captures what the prior predicts, while the revision captures what historical context adds. Their sum equals
$\mathbb{E}[\mathbf{Z}^{\mathrm{tgt}}\mid\mathcal{H}]$,
and the tower property gives
$\mathbb{E}[\boldsymbol{\Delta}\mid\mathcal{G}]=\mathbf{0}$.
This is a conditional-mean restriction, not an independence assumption.

\paragraph{Observed displacement and learning target.}
The displacement observed during representation learning decomposes as
\begin{equation}
\label{eq:formulation_displacement}
\begin{aligned}
    \mathbf{D}
    &:=
    \mathbf{Z}^{\mathrm{tgt}}-\mathbf{Z}^{\mathrm{anc}}
    =
    \boldsymbol{\Delta}+\boldsymbol{\varepsilon}^{\mathrm{lat}},
    \\
    \boldsymbol{\varepsilon}^{\mathrm{lat}}
    &:=
    \mathbf{Z}^{\mathrm{tgt}}
    -
    \mathbb{E}[\mathbf{Z}^{\mathrm{tgt}}\mid\mathcal{H}].
\end{aligned}
\end{equation}
Since
$\mathbb{E}[\boldsymbol{\varepsilon}^{\mathrm{lat}}\mid\mathcal{H}]
=\mathbf{0}$,
the optimal squared-loss predictor of $\mathbf{D}$ using $\mathcal{H}$ is
$\mathbb{E}[\mathbf{D}\mid\mathcal{H}]=\boldsymbol{\Delta}$.
Learning therefore targets the predictable component of the observed displacement. This requires neither the prior to span the true factor space nor the revision to identify a particular omitted factor. Figure~\ref{fig:synthetic_omitted_span} illustrates this decomposition and how historical context refines the prior prediction toward the conditional mean.
\begin{figure}[!htbp]
    \centering
    \includegraphics[width=0.95\linewidth]{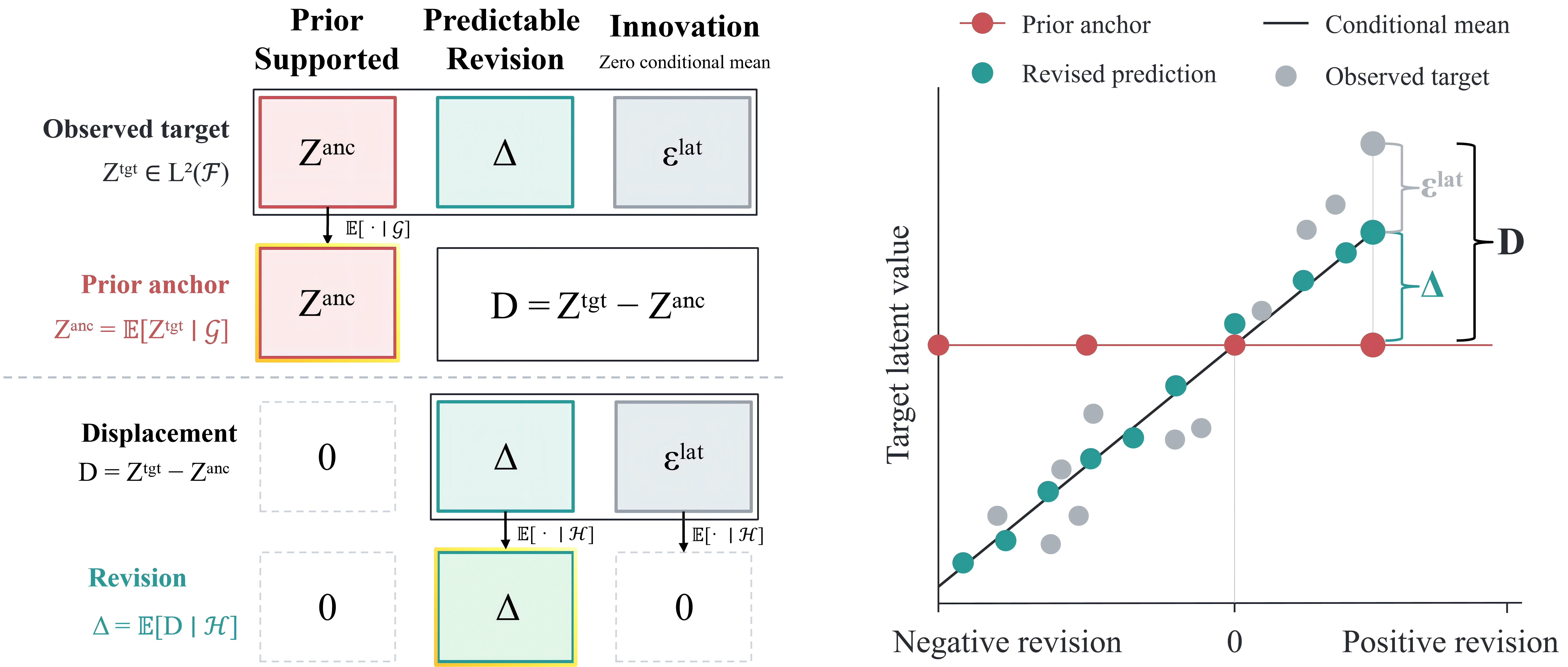}
    \caption{
    \textbf{Prior-relative predictability in latent space.}
    Left: Nested conditional projections separate the prior-supported
    anchor, predictable revision, and innovation. Conditioning the
    anchor-relative displacement on $\mathcal{H}$ yields
    $\boldsymbol{\Delta}=\mathbb{E}[\mathbf{D}\mid\mathcal{H}]$.
    Right: For a fixed prior anchor, revised predictions use historical
    context to approximate the conditional mean, while observed targets
    also contain innovations. Smaller expected squared distance to this
    mean yields lower squared prediction risk.
    }
    \label{fig:synthetic_omitted_span}
\end{figure}

\begin{proposition}[Optimal risk reduction with richer information]
\label{prop:prior_relative_projection}
For a sub-$\sigma$-algebra $\mathcal{C}\subseteq\mathcal{F}$, let
$L^2(\mathcal{C})$ denote square-integrable, $\mathcal{C}$-measurable
$\mathbb{R}^{d_z}$-valued predictors. Consider the successive objectives
\begin{equation}
\label{eq:population_separated_objectives}
\begin{aligned}
    \mathbf{a}^{\star}
    &:=
    \operatorname*{arg\,min}_{\mathbf{a}\in L^2(\mathcal{G})}
    \mathbb{E}\!\left[
        \|\mathbf{Z}^{\mathrm{tgt}}-\mathbf{a}\|_2^2
    \right],
    \\
    \mathbf{r}^{\star}
    &:=
    \operatorname*{arg\,min}_{\mathbf{r}\in L^2(\mathcal{H})}
    \mathbb{E}\!\left[
        \|\mathbf{Z}^{\mathrm{tgt}}-\mathbf{a}^{\star}-\mathbf{r}\|_2^2
    \right].
\end{aligned}
\end{equation}
Their solutions are unique up to almost-sure equality:
\begin{equation}
\label{eq:population_separated_solutions}
    \mathbf{a}^{\star}=\mathbf{Z}^{\mathrm{anc}},
    \qquad
    \mathbf{r}^{\star}=\boldsymbol{\Delta}.
\end{equation}
Writing $\mathcal{R}_{\mathcal{G}}$ and $\mathcal{R}_{\mathcal{H}}$
for the corresponding minimum population risks (expected squared prediction errors over the data distribution),
we have
\begin{equation}
\label{eq:projection_risk_decomposition}
\boxed{
    \mathcal{R}_{\mathcal{H}}
    =
    \mathcal{R}_{\mathcal{G}}
    -
    \mathbb{E}\!\left[\|\boldsymbol{\Delta}\|_2^2\right]
    \leq
    \mathcal{R}_{\mathcal{G}}.
}
\end{equation}
Equality holds if and only if
$\boldsymbol{\Delta}=\mathbf{0}$ almost surely.
\end{proposition}

\textbf{Thus, the optimal population revision captures the additional predictability that historical context provides beyond the prior, reducing squared prediction risk by exactly the non-negative
quantity $\mathbb{E}[\|\boldsymbol{\Delta}\|_2^2]$}. When history adds no conditional-mean predictability, the optimal revision is zero and the prior prediction is recovered. The non-increase guarantee also holds for parameterized revision classes containing the zero function, with the target encoder and anchor held fixed. See Appendix~\ref{app:population_proof} for the proof and
Section~\ref{sec:Methodology} for the finite-dimensional
\textsc{Stock-JEPA} model.

\section{Methodology}
\label{sec:Methodology}

\textsc{Stock-JEPA} in Figure~\ref{fig:stock_jepa_overview} learns future-state representations by combining a structured financial prior with a prediction of what historical context can add. Following Section~\ref{sec:formulation}, it first constructs financial prior statistics, then learns a latent anchor and its predictable revision. A downstream return readout evaluates the resulting features with all representation modules frozen.
\begin{figure}[htb] 
    \centering
    \includegraphics[width=0.95\linewidth]{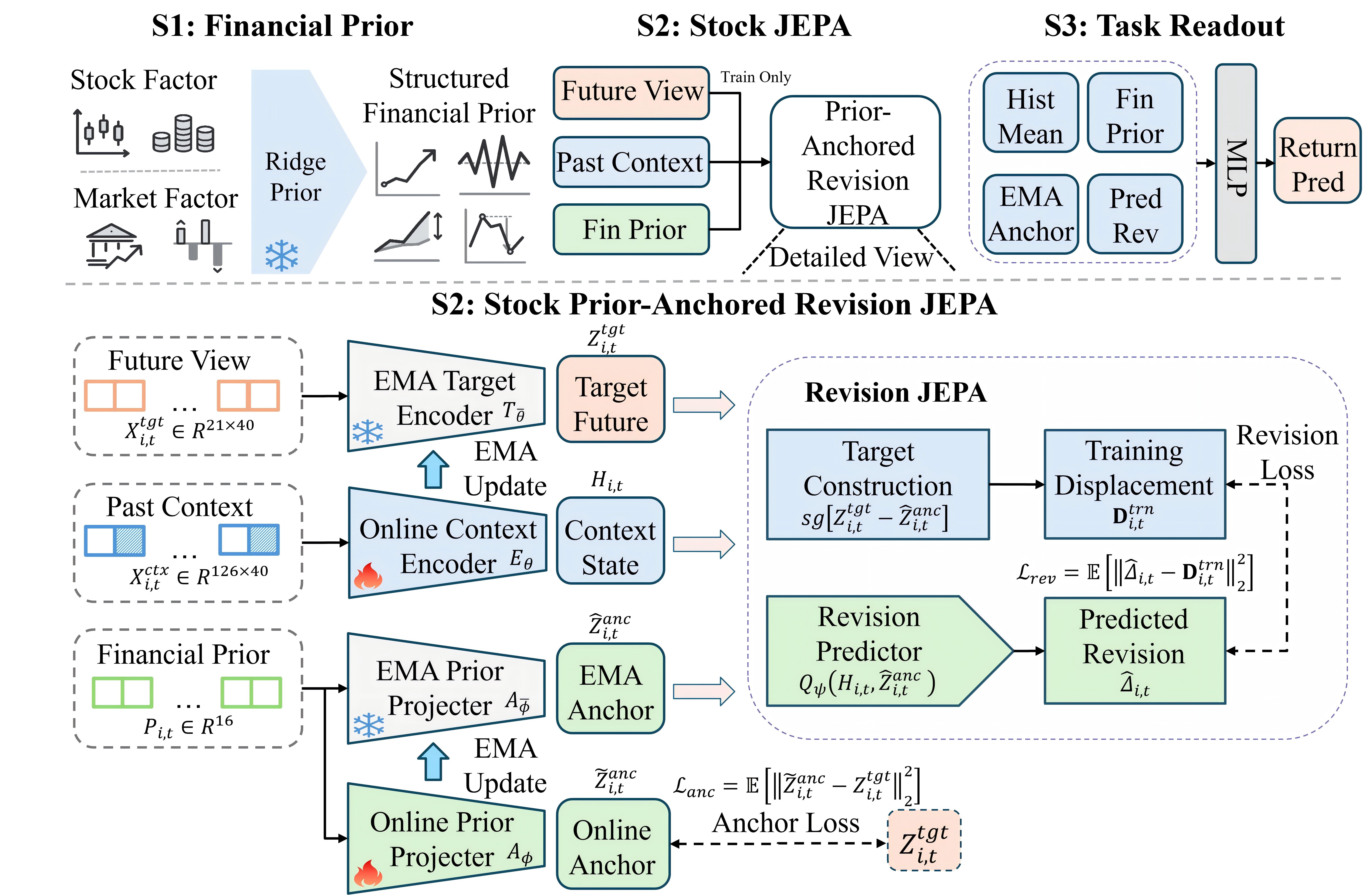} 
    \caption{ \textbf{Overview of \textsc{Stock-JEPA}.} \textbf{S1: Financial prior construction.} A financial model produces fixed prior statistics summarizing multi-horizon return and risk. \textbf{S2: Prior-anchored representation learning.} The target encoder and prior projector form future targets and financial anchors in a shared latent space. The context encoder and revision predictor combine historical information with the anchor to predict incremental revisions. Separate losses train the anchor and revision branches. \textbf{S3: Downstream return readout.} A trainable MLP combines the historical summary, prior statistics, anchor, and predicted revision, with all representation modules frozen. }
    \label{fig:stock_jepa_overview} 
\end{figure}

\subsection{Point-in-Time Structured Financial Prior}
\label{sec:structured_prior}
We construct a factor-model-inspired point-in-time baseline as a parsimonious implementation of the structured financial prior in Section~\ref{sec:formulation}. The \emph{financial prior model} uses multi-output Ridge regression to map stock and market descriptors available at time $t$ to return, volatility, drawdown, and equal-weighted market-relative return over 5- and 21-trading-day horizons (approximately one week and one month, respectively). These eight forecasts summarize financially meaningful aspects of the future window and provide the financial content of the prior. We combine the transformed point forecasts $\widehat{\mathbf{y}}_{i,t}^{\mathrm{pri}}\in\mathbb{R}^{8}$ with their coordinate-wise historical log residual scales $\boldsymbol{\sigma}_{b(t)}^{\mathrm{res}}\in\mathbb{R}^{8}$:
\begin{equation}
\label{eq:structured_prior}
    \mathbf{P}_{i,t}
    =\left[\widehat{\mathbf{y}}_{i,t}^{\mathrm{pri}};
    \boldsymbol{\sigma}_{b(t)}^{\mathrm{res}}\right]
    \in\mathbb{R}^{16}.
\end{equation}
Here, $b(t)$ denotes the forecasting block. Each block's model and residual scales use only samples whose complete outcomes are available before the block begins. The scales summarize historical fitting errors and are shared within the block; they are not estimates of sample-specific predictive uncertainty. The prior statistics are precomputed and remain fixed during representation learning.  These statistics define the structured information set $\mathcal{G}=\sigma(\mathbf{P})$. The prior projector introduced below maps this fixed prior into the target latent space, providing a finite-dimensional approximation to $\mathbb{E}[\mathbf{Z}^{\mathrm{tgt}}\mid\mathcal{G}]$ in Equation~\ref{eq:formulation_revision}. Details are provided in Appendix~\ref{app:financial_state}.

\subsection{Prior-Anchored Representation Learning}
\label{sec:revision_jepa}

The detailed S2 architecture in Figure~\ref{fig:stock_jepa_overview} follows the information structure of the formulation: the anchor branch predicts from $\mathcal{G}$, while the revision branch uses the richer history in $\mathcal{H}$. We distinguish the population anchor $\mathbf{Z}_{i,t}^{\mathrm{anc}}$ from its online estimate $\widetilde{\mathbf{Z}}_{i,t}^{\mathrm{anc}}$ and stable estimate $\widehat{\mathbf{Z}}_{i,t}^{\mathrm{anc}}$. Similarly, $\widehat{\boldsymbol{\Delta}}_{i,t}$ denotes the learned prediction of the population revision $\boldsymbol{\Delta}_{i,t}$.

\paragraph{Future target encoding and anchor projection.}
The \emph{future target encoder} $T_{\bar{\theta}}$ defines the target representation from the standardized future window $\mathbf{X}_{i,t}^{\mathrm{tgt}}\in\mathbb{R}^{21\times40}$. The \emph{financial prior projector} $A_{\phi}$ maps the prior into the same $d_z$-dimensional space:
\begin{equation}
\label{eq:target_anchor_learning}
    \mathbf{Z}_{i,t}^{\mathrm{tgt}}
    =\operatorname{sg}\!\left[
    T_{\bar{\theta}}(\mathbf{X}_{i,t}^{\mathrm{tgt}})
    \right],
    \qquad
    \widetilde{\mathbf{Z}}_{i,t}^{\mathrm{anc}}
    =A_{\phi}(\mathbf{P}_{i,t}).
\end{equation}
Here, $\operatorname{sg}$ denotes stop-gradient, and the future window is used only during representation learning. Restricting the projector to $\mathbf{P}_{i,t}$ makes its output $\mathcal{G}$-measurable. For a fixed target encoder, the corresponding unrestricted squared-loss optimum is the prior conditional mean $\mathbf{Z}_{i,t}^{\mathrm{anc}}=\mathbb{E}[\mathbf{Z}_{i,t}^{\mathrm{tgt}}\mid\mathcal{G}]$. The revision branch uses $\widehat{\mathbf{Z}}_{i,t}^{\mathrm{anc}}=\operatorname{sg}[A_{\bar{\phi}}(\mathbf{P}_{i,t})]$ as a stable estimate of this same anchor.

\paragraph{Historical context encoding and revision prediction.}
The \emph{historical context encoder} $E_{\theta}$ summarizes the preceding 126 trading days, $\mathbf{X}_{i,t}^{\mathrm{ctx}}\in\mathbb{R}^{126\times40}$. The context and target encoders share a Transformer architecture that accommodates both window lengths. The \emph{revision predictor} $Q_{\psi}$ combines the historical representation with the anchor:
\begin{equation}
\label{eq:revision_learning}
    \mathbf{H}_{i,t}
    =E_{\theta}\!\left(
    \mathcal{M}(\mathbf{X}_{i,t}^{\mathrm{ctx}})
    \right),
    \qquad
    \widehat{\boldsymbol{\Delta}}_{i,t}
    =Q_{\psi}\!\left(
    \mathbf{H}_{i,t},
    \widehat{\mathbf{Z}}_{i,t}^{\mathrm{anc}}
    \right).
\end{equation}
The operator $\mathcal{M}$ applies temporal and feature masking during training and is the identity at inference. Both windows contain 28 asset-level and 12 market-level variables. At inference, the prediction thus uses only information available in $\mathcal{H}$.

The observed training displacement is $\mathbf{D}_{i,t}^{\mathrm{trn}}=\operatorname{sg}[\mathbf{Z}_{i,t}^{\mathrm{tgt}}-\widehat{\mathbf{Z}}_{i,t}^{\mathrm{anc}}]$, the learned-anchor counterpart of $\mathbf{D}_{i,t}$ in the formulation. The identity $\mathbb{E}[\mathbf{D}_{i,t}\mid\mathcal{H}]=\boldsymbol{\Delta}_{i,t}$ following Equation~\ref{eq:formulation_displacement} motivates using this noisy displacement as a regression target: the desired prediction is its context-supported component. With an estimated anchor, the training displacement can also contain anchor estimation error.

\paragraph{Separate anchor and revision objectives.}
Proposition~\ref{prop:prior_relative_projection} assigns distinct prediction roles to the two branches: estimating what the prior
supports and estimating the remaining predictable displacement. The corresponding anchor and revision objectives are:
\begin{equation}
\label{eq:representation_objectives}
\begin{aligned}
\mathcal{L}_{\mathrm{anc}}
&=\mathbb{E}\!\left[\left\|
\widetilde{\mathbf{Z}}_{i,t}^{\mathrm{anc}}
-\mathbf{Z}_{i,t}^{\mathrm{tgt}}
\right\|_2^2\right],
&
\mathcal{L}_{\mathrm{rev}}
&=\mathbb{E}\!\left[\left\|
\widehat{\boldsymbol{\Delta}}_{i,t}
-\mathbf{D}_{i,t}^{\mathrm{trn}}
\right\|_2^2\right].
\end{aligned}
\end{equation}
The anchor loss updates only $A_{\phi}$, while the revision loss updates only $E_{\theta}$ and $Q_{\psi}$. The anchor is detached in both the predictor input and the displacement target, preventing revision gradients from directly updating the prior projector.
A loss on the sum of the anchor and revision forecasts alone would leave the allocation between the two branches unspecified.

For stable training targets, $T_{\bar{\theta}}$ and $A_{\bar{\phi}}$ are updated by exponential moving averages (EMA) of $E_{\theta}$ and $A_{\phi}$, respectively, and receive no gradients. The implemented revision loss compares normalized forecasts with standardized and clipped displacement targets; forecasts are then converted back to the original latent coordinates for downstream use. The proposition characterizes ideal targets for a fixed encoder and unrestricted predictors; the learned modules provide finite dimensional approximations as the target space evolves. Architecture and optimization details are given in Appendix~\ref{app:implementation}.

\subsection{Return Readout with Frozen Representations}
\label{sec:downstream}

The downstream readout maps the prior-anchored representations to return forecasts. The historical summary $\mathbf{s}_{i,t}$ and financial prior $\mathbf{P}_{i,t}$ provide prior features. The anchor expresses the prior in the target latent coordinates, while the predicted revision $\widehat{\boldsymbol{\Delta}}_{i,t}$ summarizes the learned context-dependent adjustment relative to that anchor. A single MLP combines the four descriptors:
\begin{equation}
\label{eq:downstream_readout}
\mathbf{u}_{i,t}
=
\left[
\mathbf{s}_{i,t};
\mathbf{P}_{i,t};
\widehat{\mathbf{Z}}_{i,t}^{\mathrm{anc}};
\widehat{\boldsymbol{\Delta}}_{i,t}
\right],
\qquad
\widehat{y}_{i,t}
=
f_{\omega}(\mathbf{u}_{i,t}).
\end{equation}
Here, $\mathbf{s}_{i,t}\in\mathbb{R}^{40}$ is the coordinate-wise mean of the standardized input features over the 126-day context window. Providing both the prior and its anchor exposes the financial descriptors and their learned latent representation to the readout. The non-linear MLP can use these prior features to condition how the predicted revision contributes to the return forecast.

After selecting the representation checkpoint, we freeze representation modules and train only the MLP on 21-day log-return labels. This separates representation learning from return-supervised readout training and evaluates the utility of fixed features. At inference, the readout requires only historical observations and precomputed prior statistics available at the forecast origin. Its forecasts provide scores for cross-sectional ranking and portfolio management. Details are given in Appendix~\ref{app:downstream}.

\section{Experiment}
\label{sec:Experiment}

\subsection{Experimental Setup}
\label{sec:experimental_setup}
\paragraph{Datasets and experimental setup.}
\label{sec:training_configuration}
We conduct experiments on the \textsc{CN\_ALL} and \textsc{US\_ALL} datasets collected from Wind and Norgate Data, respectively, covering 5,669 and 10,152 historically tradable securities in the China and U.S. stock markets. We use 2014–2015 for initial prior fitting, 2016–2021 for representation and readout training, 2022–2023 for validation, and January 2024–June 2026 for testing. We exclude cross-boundary target windows and use training-only normalization statistics. Experiments run on a server with eight NVIDIA RTX 5090 GPUs. Following the expanding-window estimation principle of \citet{gu2020empirical}, financial priors use historical outcomes and remain fixed within each prediction block. After representation learning, we freeze the prior and representation modules and train only the MLP return readout, selecting checkpoints on validation data. Dataset details and implementation settings are provided in Appendices~\ref{app:datasets} and~\ref{app:implementation}.

\subsection{Cross-Sectional Ranking and Portfolio Evaluation}
\label{sec:portfolio_evaluation}

\paragraph{Baselines and evaluation.}
We compare \textsc{Stock-JEPA} with 13 baselines from 4 model families. Traditional financial models include FMB-style \citep{fama1973risk} and IPCA \citep{kelly2019characteristics}. General time-series models include LSTM \citep{hochreiter1997long}, Transformer \citep{vaswani2017attention}, iTransformer \citep{liu2024itransformer}, and TimeMixer \citep{wang2024timemixer}. Finance-specific models include FactorVAE \citep{duan2022factorvae}, StockMixer \citep{fan2024stockmixer}, MASTER \citep{li2024master}, and PRISM-VQ \citep{kim2026prismvq}. JEPA baselines include T-JEPA \citep{thimonier2025tjepa}, TS-JEPA \citep{ennadir2024joint}, and Fin-JEPA \citep{wang2026finjepa}. All baselines are retrained using predefined hyperparameter
configurations, with common validation and test periods and checkpoints selected by validation loss. Evaluation uses a common 21-day horizon and reports Pearson IC, Spearman RankIC, and portfolio performance. We form equal-weighted top-30 long-only batches in China \citep{li2024master} and top-30/bottom-30 long--short batches in the U.S. Methods share market-specific rules, charging 5 bps \citep{conlon2025factorcovariance} per unit of executed notional at entry and exit. Portfolio metrics derive from daily closing NAV. See Appendices~\ref{app:implementation_details}
and~\ref{app:portfolio_details}.

\begin{table*}[ht]
\centering
\caption{
Out-of-sample ranking and portfolio performance
(mean$\pm$SD over three runs). Bold and underlined entries denote the best and second-best means.
}
\label{tab:portfolio_results}

{\scriptsize
\renewcommand{\arraystretch}{1.05}
\setlength{\tabcolsep}{3.2pt}

\begin{tabular*}{\textwidth}{@{\extracolsep{\fill}}clrrrrr}
\toprule
Category
& Model
& IC (\%)$\uparrow$
& RankIC (\%)$\uparrow$
& Net AR (\%)$\uparrow$
& Net Sharpe$\uparrow$
& MDD (\%)$\uparrow$ \\
\midrule

\multicolumn{7}{l}{
\textbf{Panel A: CN\_ALL -- Top-30 long-only}
} \\
\addlinespace[2pt]

\multirow{2}{*}{
    \shortstack[c]{\itshape Traditional\\\itshape financial models}
}
& FMB-style
& $5.09\pm0.00$
& $8.76\pm0.00$
& $9.51\pm0.00$
& $0.67\pm0.00$
& $\second{-16.36\pm0.00}$ \\

& IPCA
& $7.19\pm0.01$
& $8.69\pm0.02$
& $20.99\pm0.00$
& $0.78\pm0.00$
& $-27.02\pm0.00$ \\

\cmidrule(lr){2-7}

\multirow{4}{*}{
    \shortstack[c]{\itshape General\\\itshape Deep learning}
}
& LSTM
& $5.18\pm2.90$
& $8.24\pm2.49$
& $27.84\pm9.83$
& $0.95\pm0.11$
& $-27.38\pm8.96$ \\

& Transformer
& $5.82\pm2.48$
& $7.39\pm2.84$
& $40.95\pm15.35$
& $1.14\pm0.34$
& $-31.30\pm4.01$ \\

& iTransformer
& $5.39\pm1.06$
& $6.84\pm1.52$
& $24.19\pm6.17$
& $0.86\pm0.13$
& $-29.26\pm1.59$ \\

& TimeMixer
& $4.82\pm0.81$
& $6.45\pm1.30$
& $29.95\pm19.64$
& $0.92\pm0.31$
& $-27.50\pm1.52$ \\

\cmidrule(lr){2-7}

\multirow{4}{*}{
    \shortstack[c]{\itshape Finance-based\\\itshape deep models}
}
& FactorVAE
& $4.26\pm1.29$
& $6.20\pm1.73$
& $24.68\pm3.05$
& $0.87\pm0.14$
& $-24.69\pm4.97$ \\

& StockMixer
& $2.93\pm2.40$
& $5.90\pm3.03$
& $23.37\pm19.48$
& $0.89\pm0.77$
& $-21.96\pm6.59$ \\

& MASTER
& $\second{7.34\pm2.01}$
& $\second{9.84\pm1.08}$
& $\second{41.81\pm4.00}$
& $\second{1.23\pm0.25}$
& $-21.46\pm2.79$ \\

& PRISM-VQ
& $2.10\pm0.73$
& $5.19\pm0.44$
& $8.19\pm3.17$
& $0.70\pm0.19$
& $\best{-13.64\pm5.75}$ \\

\cmidrule(lr){2-7}

\multirow{3}{*}{
    \shortstack[c]{\itshape JEPA\\\itshape baselines}
}
& T-JEPA
& $0.89\pm1.10$
& $1.92\pm1.98$
& $18.03\pm13.72$
& $0.35\pm0.23$
& $-32.08\pm4.13$ \\

& TS-JEPA
& $2.08\pm2.34$
& $4.09\pm1.49$
& $2.75\pm6.10$
& $0.21\pm0.40$
& $-16.94\pm3.61$ \\

& Fin-JEPA
& $1.15\pm1.83$
& $1.61\pm1.74$
& $19.03\pm4.86$
& $0.66\pm0.07$
& $-25.38\pm4.69$ \\

\cmidrule(lr){2-7}

\textit{Ours}
& \textbf{Stock-JEPA}
& $\best{10.70\pm0.61}$
& $\best{11.22\pm0.24}$
& $\best{45.66\pm5.12}$
& $\best{1.50\pm0.28}$
& $-31.09\pm2.30$ \\

\midrule

\multicolumn{7}{l}{
\textbf{Panel B: US\_ALL -- Top-30 minus bottom-30 long--short}
} \\
\addlinespace[2pt]

\multirow{2}{*}{
    \shortstack[c]{\itshape Traditional\\\itshape financial models}
}
& FMB-style
& $9.44\pm0.00$
& $7.61\pm0.00$
& $49.04\pm0.00$
& $\second{1.11\pm0.00}$
& $-57.98\pm0.00$ \\

& IPCA
& $9.64\pm0.00$
& $7.95\pm0.00$
& $41.53\pm0.01$
& $0.76\pm0.00$
& $-68.44\pm0.00$ \\

\cmidrule(lr){2-7}

\multirow{4}{*}{
    \shortstack[c]{\itshape General\\\itshape Deep learning}
}
& LSTM
& $11.30\pm0.57$
& $10.59\pm1.41$
& $40.52\pm19.18$
& $0.60\pm0.25$
& $-69.19\pm3.07$ \\

& Transformer
& $8.82\pm2.94$
& $8.92\pm1.85$
& $10.38\pm29.16$
& $0.20\pm0.64$
& $-66.02\pm6.33$ \\

& iTransformer
& $4.37\pm3.81$
& $3.43\pm2.82$
& $-3.26\pm44.22$
& $-0.09\pm0.99$
& $-72.23\pm19.36$ \\

& TimeMixer
& $10.16\pm1.91$
& $9.88\pm1.57$
& $45.67\pm13.83$
& $0.69\pm0.35$
& $-70.41\pm16.77$ \\

\cmidrule(lr){2-7}

\multirow{4}{*}{
    \shortstack[c]{\itshape Finance-based\\\itshape deep models}
}
& FactorVAE
& $\best{12.85\pm0.59}$
& $\second{11.18\pm2.28}$
& $\second{60.66\pm19.79}$
& $1.03\pm0.66$
& $-66.13\pm16.26$ \\

& StockMixer
& $7.46\pm3.33$
& $7.26\pm2.60$
& $32.70\pm25.70$
& $0.43\pm0.68$
& $-72.38\pm17.37$ \\

& MASTER
& $9.76\pm2.14$
& $9.08\pm3.62$
& $45.05\pm30.05$
& $1.00\pm0.57$
& $-43.31\pm9.01$ \\

& PRISM-VQ
& $9.01\pm3.06$
& $8.26\pm3.82$
& $46.24\pm65.03$
& $0.71\pm1.03$
& $-65.14\pm26.41$ \\

\cmidrule(lr){2-7}

\multirow{3}{*}{
    \shortstack[c]{\itshape JEPA\\\itshape baselines}
}
& T-JEPA
& $-4.57\pm4.90$
& $-3.77\pm3.26$
& $1.49\pm2.53$
& $0.08\pm0.14$
& $\second{-42.86\pm18.46}$ \\

& TS-JEPA
& $-2.37\pm7.37$
& $-0.50\pm4.71$
& $-11.01\pm9.62$
& $-0.30\pm0.27$
& $-65.37\pm9.04$ \\

& Fin-JEPA
& $1.01\pm1.98$
& $0.69\pm1.86$
& $-10.50\pm24.90$
& $-0.37\pm0.79$
& $-50.08\pm16.01$ \\

\cmidrule(lr){2-7}

\textit{Ours}
& \textbf{Stock-JEPA}
& $\second{11.67\pm0.79}$
& $\best{12.96\pm0.48}$
& $\best{68.20\pm15.92}$
& $\best{1.24\pm0.20}$
& $\best{-40.12\pm13.12}$ \\

\bottomrule
\end{tabular*}
}
\vspace{-16pt}
\end{table*}

\paragraph{Results.}
Table~\ref{tab:portfolio_results} shows that \textsc{Stock-JEPA} ranks first in four of five metrics in each market.
RankIC reaches 11.22\% on \textsc{CN\_ALL} and 12.96\% on \textsc{US\_ALL}, exceeding MASTER and FactorVAE by 14.0\%
and 15.9\%, respectively. The strongest JEPA baselines achieve RankIC of 4.09\% in China and 0.69\% in the U.S., both well below \textsc{Stock-JEPA}. Mean net annualized returns and net Sharpe ratios are also highest in both markets: 45.66\% and 1.50 in China, and 68.20\% and 1.24 in the U.S. The broad U.S. universe includes small-cap and highly volatile stocks. Selecting 30 stocks on each side can expose individual batches to sharp price moves, particularly losses from upward jumps in short positions. Under this common U.S. portfolio protocol, \textsc{Stock-JEPA} records the least severe mean maximum drawdown ($-40.12\%$), complementing its leading return and Sharpe. Traditional financial models remain competitive: FMB-style and IPCA attain CN\_ALL RankIC of 8.76\% and 8.69\%, exceeding all general deep-learning baselines and supporting our motivation to combine financial structure with deep learning.

\subsection{Understanding the Prior-Anchored Representation}
\label{sec:Understanding representation}
Our framework aims to turn historical information into useful revisions to a financial prior. We examine whether it achieves this through three questions:
(i) do the framework's components and training choices improve downstream performance?
(ii) do revisions distinguish stocks and adjust as market conditions change?
and (iii) can known predictive signals be recovered from the anchor and revision when learned from noisy data?

\paragraph{Predictive utility through ablation studies.}
\label{sec:component_ablation}
\textit{Component ablation.}
With a linear readout, S1+S2 exceeds all standalone
configurations in mean RankIC and net Sharpe, supporting
the predictive value of prior-anchored features.
Adding the historical summary and an MLP improves both
further (Table~\ref{tab:component_ablation};
Appendix~\ref{app:component_ablation}).
\textit{Matched readout ablation.}
With frozen representations and matched MLPs, adding
revisions to the historical summary, prior, and anchor
raises mean RankIC by 2.02/3.67 percentage points on
CN\_ALL/US\_ALL.
Revisions also outperform equally sized context features
on all four mean metrics in both markets, suggesting that
prior-relative prediction makes history more useful to
the readout (Appendix~\ref{app:matched_mlp_readout}).
\textit{Controlled representation ablation.}
Under matched experimental settings, Separate-online improves
mean net AR and Sharpe over Joint-online in both markets.
EMA anchoring improves both metrics further, yielding the
highest means on both portfolio measures in each market.
Separate-EMA also outperforms Ordinary-JEPA on all four
mean metrics in both markets, with the same historical
summary and prior inputs supplied to every readout.
These gains support the value of prior-anchored representation
learning (Table~\ref{tab:component_ablation};
Appendix~\ref{app:controlled_representation_ablation}).
\vspace{-0.5\baselineskip}

\begin{wraptable}{r}{0.7\textwidth}
    \centering
    \vspace{-0.4cm}
    \caption{Selected pipeline and controlled representation ablations.}
    \vspace{-0.2cm}
    \label{tab:component_ablation}
    \scriptsize
    \setlength{\tabcolsep}{3.2pt}
    \renewcommand{\arraystretch}{1.05}
    \begin{tabular}{lcccc}
        \toprule
        & \multicolumn{2}{c}{CN\_ALL}
        & \multicolumn{2}{c}{US\_ALL} \\
        \cmidrule(lr){2-3}\cmidrule(lr){4-5}
        Configuration
        & RankIC (\%) $\uparrow$ & Net Sharpe $\uparrow$
        & RankIC (\%) $\uparrow$ & Net Sharpe $\uparrow$ \\
        \midrule
        \multicolumn{5}{l}{\textit{Pipeline configurations}} \\
        S2-only
        & $6.21\pm3.17$ & $0.85\pm0.27$
        & $7.92\pm2.80$ & $0.81\pm0.10$ \\
        S1+S2
        & $9.94\pm1.23$ & $1.10\pm0.27$
        & $11.89\pm0.66$ & $1.00\pm0.19$ \\
        Full (S1+S2+S3)
        & $\mathbf{11.22\pm0.24}$ & $\mathbf{1.50\pm0.28}$
        & $\mathbf{12.96\pm0.48}$ & $\mathbf{1.24\pm0.20}$ \\
        \midrule
        \multicolumn{5}{l}{
            \textit{Controlled representation continuation}} \\
        Separate-EMA (ours)
        & $11.11\pm0.94$ & $\mathbf{1.51\pm0.26}$
        & $\mathbf{12.45\pm0.74}$ & $\mathbf{1.27\pm0.13}$ \\
        Separate-online
        & $11.13\pm0.68$ & $1.48\pm0.26$
        & $12.14\pm0.75$ & $1.13\pm0.10$ \\
        Joint-online
        & $\mathbf{11.49\pm0.60}$ & $1.25\pm0.12$
        & $12.37\pm1.20$ & $1.09\pm0.20$ \\
        Ordinary-JEPA
        & $9.17\pm0.59$ & $1.01\pm0.23$
        & $9.05\pm0.90$ & $0.50\pm0.09$ \\
        \bottomrule
    \end{tabular}
    \vspace{-5pt}
\end{wraptable}

\paragraph{Representation structure and market responses.}
\label{sec:representation_structure}
After removing each day's cross-sectional mean, the 128-dimensional revisions have a mean off-diagonal absolute correlation of approximately 0.30 and an average daily effective rank of 21.1 (Figure~\ref{fig:revision_structure}). These statistics indicate that differences among stocks span multiple latent directions. See Appendix~\ref{app:revision_structure} for details. The market-average revision direction also changes across dates, while large daily adjustments occur across stocks around monetary-policy repricing and trade tensions (Figure~\ref{fig:revision_response_recovery}, left; Appendix~\ref{app:revision_response}).
The two views separate differences among stocks on the same date from changes in the market-average representation across dates. The event statistic uses the median stock-level change, so its peaks reflect broad adjustments with limited influence from a few extreme observations. Together, these patterns show that revisions supply stock-level distinctions for ranking alongside broad temporal adjustments. This fits their intended role: adjusting the financial prior to reflect both the stock and its changing market context.
\begin{figure}[H]
   \vspace{-12pt}
    \centering
    \includegraphics[width=0.75\linewidth]{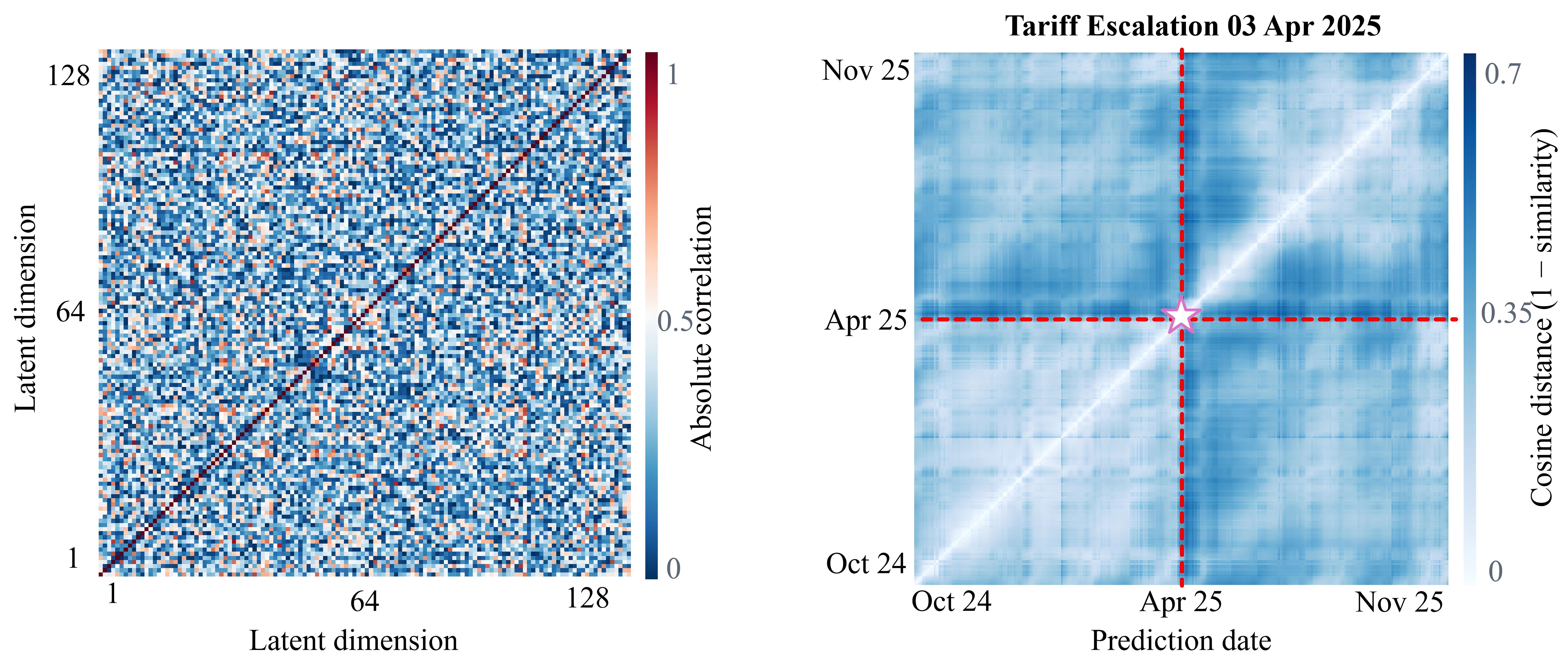}
    \caption{
    \textbf{Cross-sectional and temporal revision structure on US\_ALL.}
    Left: Dimension-wise absolute correlations computed from equally averaged daily cross-sectional covariance matrices. Right: Cross-date cosine distances between market-average revisions after training-set standardization, averaged over three seeds. The star marks April 3, 2025.
    }
    \label{fig:revision_structure}
    \label{fig:representation_diagnostics}
    \vspace{-16pt}
\end{figure}

\paragraph{Predictable signal recovery in simulation.}
\label{sec:simulation_recovery}
In simulation, the anchor and revision retain complementary slow and fast return signals without component-level supervision during representation learning. We use known component labels only to fit and evaluate Ridge probes on frozen features (Figure~\ref{fig:revision_response_recovery}, right). For Separate-EMA, the anchor achieves a mean test $R^2$ of 66.14\% on the slow signal, while the revision achieves 49.25\% on the fast
signal. Their concatenation retains both signals (62.22\% slow and 55.24\% fast). Ordinary-JEPA and direct supervision also recover the fast signal, but neither achieves positive mean slow-signal recovery. For independent future noise, all configurations score below the constant reference on average (Appendix~\ref{app:simulation_recovery}).

Although the slow state is more persistent, its return contribution depends on both the state and asset-specific exposures; persistence alone does not ensure linear recovery from learned features. A related finding in asset pricing is that incorporating expected-return information helps recover economically relevant factors missed by covariance-based extraction \citep{lettau2020factors}. Our results suggest that the calibrated prior helps preserve persistent return structure, while revision learning retains complementary faster signals. This role for financial structure is consistent with the forecasting benefits of economically motivated restrictions
\citep{campbell2008predicting}. Together, these findings support the idea behind Equation~\ref{eq:formulation_displacement}: the anchor predicts from the prior, while the revision uses historical context to learn predictable corrections from noisy future targets.
\begin{figure}[!htbp]
    \vspace{-6pt}
    \centering
    \includegraphics[width=0.85\linewidth]{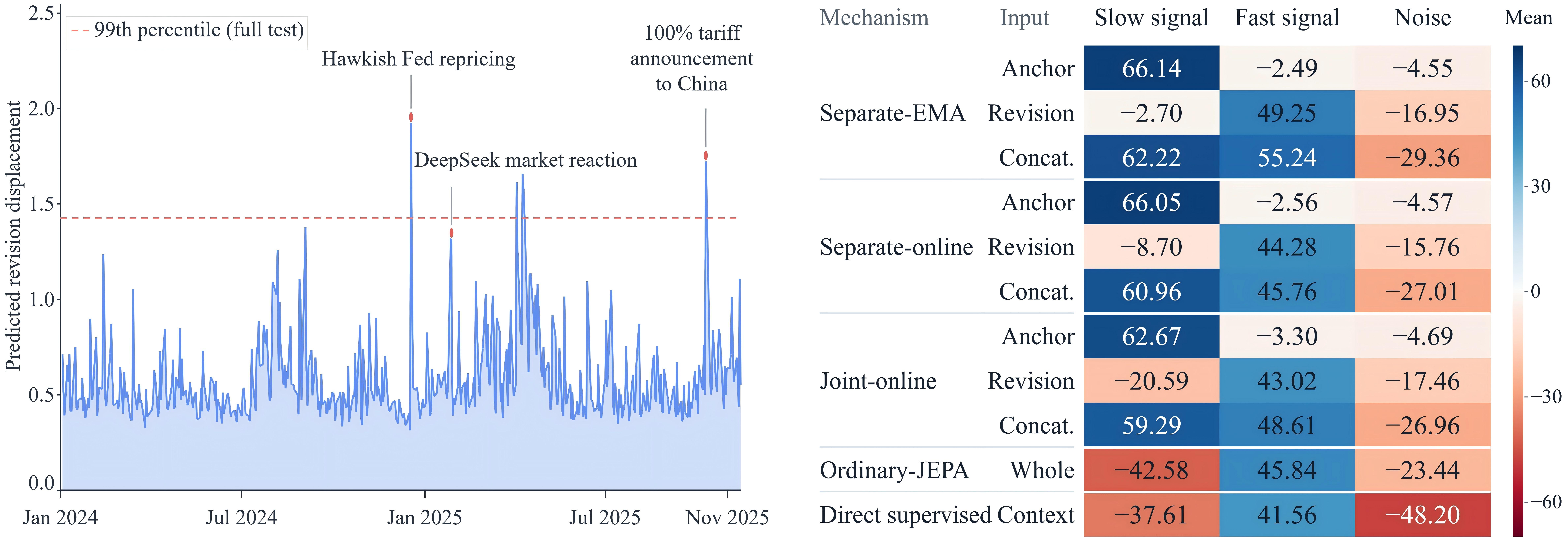}
    \caption{ \textbf{Revision responses and recovery of known signals.} Left: Cross-sectional median of standardized daily revision changes on US\_ALL, averaged over three runs. Right: Frozen Ridge-probe test $R^2$ (\%) for slow and fast return signals and independent future noise.} \label{fig:revision_response_recovery}
    \vspace{-16pt}
\end{figure}

\section{Conclusion}
\label{sec:conclusion}
We introduced \textsc{Stock-JEPA}, which learns prior-relative revisions through separate anchor and revision objectives. Population analysis characterizes the risk reduction from historical context. Frozen representations lead evaluated methods in mean RankIC and net Sharpe in both markets. Ablations support incremental predictive value; diagnostics show stock-level and temporal variation, and simulation demonstrates complementary signal recovery from noisy targets.

\section*{AI Use Statement}
We used generative AI tools to improve the grammar, wording,
and readability of the manuscript. The authors reviewed and
revised the AI-assisted text and take full responsibility for
the accuracy and integrity of the final manuscript.

\bibliography{iclr2027_conference}
\bibliographystyle{iclr2027_conference}

\appendix

\section{Population Characterization}
\label{app:population_proof}

\subsection{Proof of Proposition~\ref{prop:prior_relative_projection}}

\begin{proof}
Equip
$L^2(\Omega,\mathcal{F},\mathbb{P};\mathbb{R}^{d_z})$
with the inner product
\begin{equation}
\label{eq:l2_inner_product}
    \langle \mathbf{U},\mathbf{V}\rangle_{L^2}
    :=
    \mathbb{E}[\mathbf{U}^{\top}\mathbf{V}].
\end{equation}
The spaces $L^2(\mathcal{G})$ and $L^2(\mathcal{H})$ are closed linear subspaces of this Hilbert space. Conditional expectation with respect to either sigma-algebra is the corresponding orthogonal projection.

For the first objective, the projection of
$\mathbf{Z}^{\mathrm{tgt}}$ onto $L^2(\mathcal{G})$ is
\begin{equation}
\label{eq:proof_anchor_projection}
    \mathbf{a}^{\star}
    =
    \mathbb{E}[\mathbf{Z}^{\mathrm{tgt}}\mid\mathcal{G}]
    =
    \mathbf{Z}^{\mathrm{anc}}.
\end{equation}
For every $\mathbf{a}\in L^2(\mathcal{G})$, the Pythagorean identity gives
\begin{equation}
\label{eq:proof_anchor_pythagorean}
\begin{aligned}
    \mathbb{E}\!\left[
        \|\mathbf{Z}^{\mathrm{tgt}}-\mathbf{a}\|_2^2
    \right]
    ={}&
    \mathbb{E}\!\left[
        \|\mathbf{Z}^{\mathrm{tgt}}-\mathbf{a}^{\star}\|_2^2
    \right]
    \\
    &+
    \mathbb{E}\!\left[
        \|\mathbf{a}^{\star}-\mathbf{a}\|_2^2
    \right].
\end{aligned}
\end{equation}
The second term is nonnegative and vanishes if and only if
$\mathbf{a}=\mathbf{a}^{\star}$ almost surely. Hence the first minimizer is unique up to almost-sure equality.

Because $\mathcal{G}\subseteq\mathcal{H}$,
$\mathbf{a}^{\star}$ is also $\mathcal{H}$-measurable.
The second objective projects
$\mathbf{Z}^{\mathrm{tgt}}-\mathbf{a}^{\star}$
onto $L^2(\mathcal{H})$. Its almost-surely unique minimizer is
\begin{equation}
\label{eq:proof_revision_projection}
\begin{aligned}
    \mathbf{r}^{\star}
    &=
    \mathbb{E}[
        \mathbf{Z}^{\mathrm{tgt}}-\mathbf{a}^{\star}
        \mid\mathcal{H}
    ]
    \\
    &=
    \mathbb{E}[\mathbf{Z}^{\mathrm{tgt}}\mid\mathcal{H}]
    -
    \mathbb{E}[\mathbf{Z}^{\mathrm{tgt}}\mid\mathcal{G}]
    =
    \boldsymbol{\Delta}.
\end{aligned}
\end{equation}

The tower property further implies
\begin{equation}
\label{eq:proof_revision_conditional_mean}
\begin{aligned}
    \mathbb{E}[\boldsymbol{\Delta}\mid\mathcal{G}]
    &=
    \mathbb{E}\!\left[
        \mathbb{E}[\mathbf{Z}^{\mathrm{tgt}}\mid\mathcal{H}]
        \mid\mathcal{G}
    \right]
    -
    \mathbb{E}[\mathbf{Z}^{\mathrm{tgt}}\mid\mathcal{G}]
    \\
    &=\mathbf{0}.
\end{aligned}
\end{equation}
Since
$\mathbf{D}=\mathbf{Z}^{\mathrm{tgt}}-\mathbf{Z}^{\mathrm{anc}}$,
we also obtain
\begin{equation}
\label{eq:proof_displacement_predictor}
    \mathbb{E}[\mathbf{D}\mid\mathcal{H}]
    =
    \mathbb{E}[\mathbf{Z}^{\mathrm{tgt}}\mid\mathcal{H}]
    -
    \mathbf{Z}^{\mathrm{anc}}
    =
    \boldsymbol{\Delta}.
\end{equation}

\paragraph{Projection-risk decomposition.}
The corresponding optimal risks are
\begin{equation}
\label{eq:optimal_projection_risks}
\begin{aligned}
    \mathcal{R}_{\mathcal{G}}
    &:=
    \mathbb{E}\!\left[
        \|\mathbf{Z}^{\mathrm{tgt}}
        -\mathbf{Z}^{\mathrm{anc}}\|_2^2
    \right],
    \\
    \mathcal{R}_{\mathcal{H}}
    &:=
    \mathbb{E}\!\left[
        \|\mathbf{Z}^{\mathrm{tgt}}
        -\mathbf{Z}^{\mathrm{anc}}
        -\boldsymbol{\Delta}\|_2^2
    \right].
\end{aligned}
\end{equation}
Since
$\mathbf{Z}^{\mathrm{anc}}+\boldsymbol{\Delta}
=\mathbb{E}[\mathbf{Z}^{\mathrm{tgt}}\mid\mathcal{H}]$,
these are the minimum squared prediction risks under
$\mathcal{G}$ and $\mathcal{H}$, respectively.

By Equation~\ref{eq:formulation_displacement},
\begin{equation}
\label{eq:proof_displacement_decomposition}
    \mathbf{Z}^{\mathrm{tgt}}-\mathbf{Z}^{\mathrm{anc}}
    =
    \boldsymbol{\Delta}
    +
    \boldsymbol{\varepsilon}^{\mathrm{lat}}.
\end{equation}
Because $\boldsymbol{\Delta}\in L^2(\mathcal{H})$ and
$\mathbb{E}[\boldsymbol{\varepsilon}^{\mathrm{lat}}
\mid\mathcal{H}]=\mathbf{0}$,
the revision and innovation are orthogonal:
\begin{equation}
\label{eq:proof_revision_innovation_orthogonality}
    \mathbb{E}\!\left[
        \boldsymbol{\varepsilon}^{\mathrm{lat}\top}
        \boldsymbol{\Delta}
    \right]
    =
    \mathbb{E}\!\left[
        \mathbb{E}[
            \boldsymbol{\varepsilon}^{\mathrm{lat}}
            \mid\mathcal{H}
        ]^{\top}
        \boldsymbol{\Delta}
    \right]
    =
    0.
\end{equation}
Expanding the squared norm therefore gives
\[
    \mathcal{R}_{\mathcal{G}}
    =
    \mathbb{E}\!\left[\|\boldsymbol{\Delta}\|_2^2\right]
    +
    \mathbb{E}\!\left[
        \|\boldsymbol{\varepsilon}^{\mathrm{lat}}\|_2^2
    \right]
    =
    \mathbb{E}\!\left[\|\boldsymbol{\Delta}\|_2^2\right]
    +
    \mathcal{R}_{\mathcal{H}}.
\]
Rearranging yields Equation~\ref{eq:projection_risk_decomposition}
and establishes
$\mathcal{R}_{\mathcal{H}}\leq\mathcal{R}_{\mathcal{G}}$.
Since $\|\boldsymbol{\Delta}\|_2^2$ is non-negative,
equality holds if and only if
$\boldsymbol{\Delta}=\mathbf{0}$ almost surely.
Otherwise, the risk decreases strictly by
$\mathbb{E}[\|\boldsymbol{\Delta}\|_2^2]$.
Thus, additional historical information preserves the optimal prior-based prediction when it provides no conditional-mean predictive value, and improves it when it reveals predictable variation beyond the prior.
\end{proof}

\paragraph{Scope of the result.}
The projection identity assumes a fixed target encoder and unrestricted square-integrable predictors. A related feasibility argument applies to parameterized revision classes: with the target and anchor held fixed, any class containing the zero function retains the anchor-only predictor as a feasible solution, so the infimum squared prediction loss cannot exceed the anchor-only loss. The exact risk reduction in Equation~\ref{eq:projection_risk_decomposition} additionally requires attaining the conditional projections characterized above. These statements concern optimal objective values; finite samples, optimization error, and target-encoder updates affect their practical realization. Lower representation-prediction risk also does not by itself guarantee downstream economic gains. We therefore evaluate the learned revisions on large-scale China and U.S. equity datasets to assess their practical predictive and portfolio utility.

\section{Structured Financial Prior}
\label{app:financial_state}

\paragraph{Point-in-time inputs and financial targets.}
At each forecast origin $t$, \textsc{Stock-JEPA} constructs its financial
prior from a 40-dimensional descriptor vector
$\mathbf{x}_{i,t}\in\mathbb{R}^{40}$. The vector contains 28 asset-level
descriptors covering returns, momentum, volatility, price ranges, trading
activity, and liquidity, together with 12 market-level descriptors covering
market returns, volatility, dispersion, benchmark dynamics, and tradability.
All historical statistics use observations available through $t$.

For each horizon $h\in\{5,21\}$, let
$\ell_{i,t+k}=\log C_{i,t+k}-\log C_{i,t+k-1}$
denote the daily log return. For each horizon, the four financial targets are
\begin{equation}
\label{eq:app_financial_targets}
\begin{aligned}
R_{i,t}^{(h)}
&= \sum_{k=1}^{h}\ell_{i,t+k},
\\
V_{i,t}^{(h)}
&= \sqrt{
    \frac{1}{h-1}
    \sum_{k=1}^{h}
    \left(
        \ell_{i,t+k}-\frac{R_{i,t}^{(h)}}{h}
    \right)^2
},
\\
M_{i,t}^{(h)}
&= \min_{0\leq k\leq h}
    \left\{
        \log C_{i,t+k}
        -\max_{0\leq j\leq k}\log C_{i,t+j}
    \right\},
\\
E_{i,t}^{(h)}
&= R_{i,t}^{(h)}
   -\sum_{k=1}^{h}\ell^{\mathrm{EW}}_{t+k}.
\end{aligned}
\end{equation}
Here, $\ell^{\mathrm{EW}}_{t}$ denotes the panel's daily equal-weighted
market log return. Volatility is unannualized, and drawdown is represented as
a non-positive log-price decline. Concatenating the four targets at the two
horizons gives
\begin{equation}
\begin{aligned}
\mathbf{y}_{i,t}
&=
\left[
R_{i,t}^{(5)},V_{i,t}^{(5)},M_{i,t}^{(5)},E_{i,t}^{(5)},
\right.
\\[-0.2em]
&\hspace{2.9em}\left.
R_{i,t}^{(21)},V_{i,t}^{(21)},M_{i,t}^{(21)},E_{i,t}^{(21)}
\right]^{\top}
\in\mathbb{R}^{8}.
\end{aligned}
\end{equation}
Only samples with complete and finite target windows contained within a
contiguous asset segment are retained.

\paragraph{Blockwise prior estimation.}
For prediction block $b$ beginning at $\tau_b$, the historical
fitting set $\mathcal I_b$ includes only samples whose complete
21-day outcomes end before $\tau_b$.
Input features are standardized using block-specific historical
moments, with non-finite standardized entries replaced by zero
and values clipped to $[-8,8]$.
Using the resulting features $\widetilde{\mathbf{x}}^{(b)}_{i,t}$,
we fit a multi-output ridge model:
\begin{equation}
\label{eq:app_ridge_objective}
\begin{aligned}
(\widehat{\mathbf{W}}_b,\widehat{\mathbf{c}}_b)
=\arg\min_{\mathbf{W},\mathbf{c}}\;
&\sum_{(i,t)\in\mathcal{I}_b}
\left\|
\mathbf{y}_{i,t}-\mathbf{c}
-\mathbf{W}^{\top}\widetilde{\mathbf{x}}^{(b)}_{i,t}
\right\|_2^2
\\[-0.2em]
&+\alpha\|\mathbf{W}\|_F^2.
\end{aligned}
\end{equation}
Here, $\alpha=20$, and the intercept is not penalized. For a sample assigned
to block $b$, the resulting raw prior prediction is
$\widehat{\mathbf{y}}^{\mathrm{raw}}_{i,t}
=\widehat{\mathbf{c}}_b+\widehat{\mathbf{W}}_b^{\top}
\widetilde{\mathbf{x}}^{(b)}_{i,t}$.
Fitting pools exceeding three million observations are deterministically
thinned by sample ID.

The temporal fitting schedule is detailed in
Appendix~\ref{app:temporal_protocol} and illustrated
in Figure~\ref{fig:temporal_protocol}.

\paragraph{Residual scales and normalization.}
For coordinate $j$, define the fitting residual
$e_{i,t,j}=y_{i,t,j}-\widehat{y}^{\mathrm{raw}}_{i,t,j}$.
Its blockwise mean and scale are
\begin{equation}
\label{eq:app_residual_scale}
\begin{aligned}
\bar{e}_{b,j}
&=
\frac{1}{|\mathcal{I}_b|}
\sum_{(i,t)\in\mathcal{I}_b}e_{i,t,j},
\\
s^{\mathrm{res}}_{b,j}
&=
\sqrt{
\max\left\{
\frac{1}{|\mathcal{I}_b|}
\sum_{(i,t)\in\mathcal{I}_b}
(e_{i,t,j}-\bar{e}_{b,j})^2,
10^{-8}
\right\}
}.
\end{aligned}
\end{equation}
Let $m_j^{\mathrm{tr}}$ and $s_j^{\mathrm{tr}}$ be target means
and standard deviations estimated over the full training split,
and set $d_j=\max(s_j^{\mathrm{tr}},10^{-6})$.
The stored prior statistics are
\begin{equation}
\label{eq:app_prior_normalization}
\begin{aligned}
\widehat{y}^{\mathrm{pri}}_{i,t,j}
&=
\operatorname{clip}_{[-8,8]}
\left(
\frac{\widehat{y}^{\mathrm{raw}}_{i,t,j}-m_j^{\mathrm{tr}}}{d_j}
\right),
\\
\sigma^{\mathrm{res}}_{b,j}
&=
\log\left(
\max\left\{
\frac{s^{\mathrm{res}}_{b,j}}{d_j},\,10^{-4}
\right\}
\right).
\end{aligned}
\end{equation}
The log residual-scale features are shared within each block.
Ridge fitting, input scaling, and raw residual-scale estimation
use block-specific histories; output normalization uses common
full-training moments fixed across all blocks.

\paragraph{Representation-learning interface.}
Let $b(t)$ denote the prediction block containing forecast origin $t$.
The eight normalized predictions and eight log residual-scale
features form the structured prior in
Equation~\ref{eq:structured_prior}:
\begin{equation}
\label{eq:app_prior_interface}
\mathbf{P}_{i,t}
=
\left[
\widehat{\mathbf{y}}^{\mathrm{pri}}_{i,t};
\boldsymbol{\sigma}^{\mathrm{res}}_{b(t)}
\right]
\in\mathbb{R}^{16}.
\end{equation}
These statistics are precomputed, remain fixed during representation
learning, and are mapped by the prior projector into the target
latent space.

\section{Representation Architecture and Optimization}
\label{app:implementation}

\paragraph{Sequence encoders and masking.}
\textsc{Stock-JEPA} uses the same sequence-encoder architecture for the
126-day historical context and the 21-day future target window. Each
40-dimensional daily observation is linearly projected to a 128-dimensional
token. A learnable summary token and positional embeddings are then added,
followed by four pre-normalized Transformer layers. Each layer uses eight
self-attention heads, a 512-dimensional feed-forward network, GELU activation,
and dropout of 0.1. The final summary token is layer-normalized to produce a
128-dimensional representation.

During training, contiguous temporal blocks of 5--21 days are sampled until at
least 30\% of the context positions are masked, and their projected tokens are
replaced by a learnable mask token. Independently, each predefined feature
group is zeroed across the context window with probability 0.25. The six groups
cover returns and momentum, volatility, price ranges, trading activity and
liquidity, market-relative features, and market descriptors. The future target
window is not masked, and all context masking is disabled during validation and
inference.

\paragraph{Prior projector and revision predictor.}
The prior projector takes the 16-dimensional prediction--log-scale vector
defined in Appendix~\ref{app:financial_state}. It consists of a linear map from
16 to 128 dimensions, followed by LayerNorm, GELU, and a second
128-dimensional linear layer. The revision predictor concatenates the
128-dimensional context representation and the detached EMA anchor, then
applies a $256\rightarrow256\rightarrow128$ MLP with GELU and dropout of 0.1
between its linear layers. 

\paragraph{Revision-coordinate normalization.}
Revision regression uses coordinate-wise running statistics of the detached
training displacement $\mathbf{D}^{\mathrm{trn}}_{i,t}$. At update $k$, let
$\mathbf{m}_k$ and $\mathbf{v}_k$ denote the running mean and variance,
respectively, and define
\begin{equation}
\label{eq:app_revision_scale}
\mathbf{s}_k
=
\max\left\{
\sqrt{\mathbf{v}_k+10^{-4}},10^{-3}
\right\}.
\end{equation}
The maximum is applied coordinate-wise. During the first 100 updates, the
statistics are formed by averaging batch means and within-batch variances;
subsequent updates use an EMA decay of 0.99. The normalized regression target
is
\begin{equation}
\label{eq:app_normalized_revision_target}
\mathbf{d}^{\mathrm{norm}}_{i,t}
=
\operatorname{clip}_{[-8,8]}
\left(
\frac{\mathbf{D}^{\mathrm{trn}}_{i,t}-\mathbf{m}_k}
{\mathbf{s}_k}
\right),
\end{equation}
where clipping and division are coordinate-wise. Writing the predictor's
normalized output as $\widehat{\mathbf{q}}_{i,t}$, the implemented revision
loss and the revision supplied to the downstream readout are
\begin{equation}
\label{eq:implementation_revision_loss}
\begin{aligned}
\mathcal{L}_{\mathrm{rev}}^{\mathrm{impl}}
&=
\mathbb{E}\!\left[
\left\|
\widehat{\mathbf{q}}_{i,t}
-\mathbf{d}^{\mathrm{norm}}_{i,t}
\right\|_2^2
\right],
\\
\widehat{\boldsymbol{\Delta}}_{i,t}
&=
\mathbf{m}_k
+\mathbf{s}_k\odot\widehat{\mathbf{q}}_{i,t}.
\end{aligned}
\end{equation}
In the main text, $Q_{\psi}$ denotes the revision predictor including
the affine conversion back to the original latent coordinates in
Equation~\ref{eq:implementation_revision_loss}. The running statistics are updated only from training batches and remain fixed during validation and inference.

\paragraph{Optimization and checkpoint selection.}
Two AdamW optimizers separately update the online prior projector and the
context-encoder--revision-predictor pair. Both use a learning rate of
$10^{-4}$, weight decay of $0.01$, and gradient-norm clipping at 1.0. Training
uses batches of 256 samples, bfloat16 mixed precision, and four epochs with at
most 1,000 updates per epoch. The anchor and normalized revision losses receive
equal weights; additional total-latent and regularization losses are disabled.

After each optimization step, the target encoder and EMA prior projector are updated with decay rates of 0.996 and 0.99, respectively. Both EMA modules
remain in evaluation mode and receive no gradients. We select the
representation checkpoint that minimizes the sum of the anchor and normalized
revision losses on the predefined validation subset.

\section{Frozen-Representation Return Readout}
\label{app:downstream}

\paragraph{Readout features and architecture.}
After checkpoint selection, all representation modules and normalization
statistics are frozen. The historical summary is the coordinate-wise mean of
the standardized, unmasked context:
\begin{equation}
\label{eq:app_historical_summary}
\mathbf{s}_{i,t}
=
\frac{1}{126}
\sum_{\tau=0}^{125}\mathbf{x}^{\mathrm{std}}_{i,t-\tau}
\in\mathbb{R}^{40}.
\end{equation}
The readout concatenates this summary, the 16-dimensional prior, the
128-dimensional EMA anchor, and the 128-dimensional revision expressed in the
original latent coordinates:
\begin{equation}
\label{eq:app_readout_input}
\mathbf{u}_{i,t}
=
\left[
\mathbf{s}_{i,t};
\mathbf{P}_{i,t};
\widehat{\mathbf{Z}}_{i,t}^{\mathrm{anc}};
\widehat{\boldsymbol{\Delta}}_{i,t}
\right]
\in\mathbb{R}^{312}.
\end{equation}
The MLP $g_{\omega}$ predicts standardized log returns and comprises
$\operatorname{LayerNorm}(312)$, a $312\rightarrow256$ linear layer, GELU,
dropout of 0.05, and a final $256\rightarrow1$ linear layer.

\paragraph{Training objective.}
The downstream target is the 21-day cumulative log return
$y_{i,t}^{\mathrm{ret}}=R_{i,t}^{(21)}$. Let $\mu_y^{\mathrm{tr}}$ and
$\sigma_y^{\mathrm{tr}}$ denote its mean and standard deviation on the training
split, and define
$y_{i,t}^{\mathrm{norm}}
=(y_{i,t}^{\mathrm{ret}}-\mu_y^{\mathrm{tr}})/\sigma_y^{\mathrm{tr}}$.
The MLP is trained using the SmoothL1 objective
\begin{equation}
\label{eq:app_readout_loss}
\begin{aligned}
\mathcal{L}_{\mathrm{read}}
&=
\mathbb{E}\!\left[
\rho\!\left(
g_{\omega}(\mathbf{u}_{i,t})-y^{\mathrm{norm}}_{i,t}
\right)
\right],
\\
\rho(e)
&=
\begin{cases}
\frac{1}{2}e^2, & \text{if } |e|<1,\\
|e|-\frac{1}{2}, & \text{if } |e|\geq 1.
\end{cases}
\end{aligned}
\end{equation}
Each optimization batch contains the available training samples from one
prediction date, and the order of dates is shuffled in every epoch. Only the
MLP parameters are updated. We use AdamW with a learning rate of $10^{-4}$,
weight decay of $0.05$, and gradient-norm clipping at 1.0. Training runs for six
epochs with at most 1,000 updates per epoch.

\paragraph{Selection and prediction.}
We select the readout checkpoint with the lowest sample-averaged SmoothL1 loss
on the predefined validation subset. At inference, the model uses the unmasked
historical context and fixed prior statistics to construct the four readout
inputs. The complete readout $f_{\omega}$ in
Equation~\ref{eq:downstream_readout} includes the fixed inverse
standardization:
\begin{equation}
\widehat{y}_{i,t}
=
f_{\omega}(\mathbf{u}_{i,t})
=
\mu_y^{\mathrm{tr}}
+\sigma_y^{\mathrm{tr}}g_{\omega}(\mathbf{u}_{i,t}).
\end{equation}
The test set is evaluated only after selecting both the representation and
readout checkpoints.

\section{Baseline Implementation}
\label{app:implementation_details}

\paragraph{Common data and prediction task.}
The 13 baselines share the prepared stock panels, evaluation
targets, and validation and test periods with \textsc{Stock-JEPA}.
All baselines are trained on 2014--2021 data, whereas
\textsc{Stock-JEPA} uses 2014--2015 for initial S1 calibration
and 2016--2021 for representation and readout training.
Canonical inputs contain 126 trading sessions with 28 stock
features and 12 market features. The prepared training split contains samples with $\mathrm{prediction\_date}\geq$ January 1, 2014 and
$\mathrm{future\_end}<$ January 1, 2022.
The prepared validation split contains samples with
$\mathrm{prediction\_date}\geq$ January 1, 2022 and
$\mathrm{future\_end}<$ January 1, 2024.
Test forecast dates begin on January 1, 2024.
Samples whose target windows cross split boundaries
are excluded. These dates define the prepared splits.

Return-regression models predict the standardized 21-day
cumulative log return. Predictions are transformed back to
log-return units before evaluation.
Model-specific adapters organize the common inputs into
sequences, feature tokens, patches, or same-date cross-sections.

\paragraph{Baseline training configurations.}
The supervised neural baselines are trained for six epochs with
at most 1,000 updates per epoch, AdamW with learning rate
$3\times10^{-4}$, weight decay $0.01$, gradient-norm clipping
at 1.0, and a nominal batch size of 256. Model configurations
are fixed before test evaluation, and checkpoints are selected
by the lowest validation SmoothL1 loss. FactorVAE and PRISM-VQ
additionally retain their implemented training-only auxiliary
objectives.

The FMB-style baseline \citep{fama1973risk}
fits regularized cross-sectional regressions separately on each
training date and averages their coefficients for out-of-sample
prediction, with regularization selected by validation MSE.
IPCA \citep{kelly2019characteristics} estimates characteristic-based
factor loadings and latent factors by alternating least squares
on training data, selects the factor count by validation MSE,
and forecasts using the training-period mean factor rather than
factors inferred from held-out returns.

T-JEPA, TS-JEPA, and Fin-JEPA first learn representations and
then train a linear return readout with the encoder frozen.
Pretraining uses at most 12 epochs of 1,000 updates, AdamW with
learning rate $3\times10^{-4}$ and weight decay $0.04$.
The lowest validation representation loss selects the
pretraining checkpoint. The downstream readout follows the
six-epoch validation-selection protocol above.

T-JEPA forms feature tokens from the 126-session channel
histories. TS-JEPA partitions the input into 18 patches of
seven sessions. Fin-JEPA uses the last 30 context sessions
and a 10-session future representation target within the
canonical windows; its downstream return horizon remains
21 sessions.
Stock-JEPA architecture and readout settings are specified
in Appendices~\ref{app:implementation} and~\ref{app:downstream}.

\section{Dataset Construction and Temporal Protocol}
\label{app:datasets}
\subsection{Dataset construction and statistics.}
We evaluate on two broad-market equity panels, \textsc{CN\_ALL} and \textsc{US\_ALL}, sourced from Wind and Norgate Data, respectively. Covering the full available market universes of 5,669 and 10,152 historically tradable securities, including delisted stocks, these datasets combine broad cross-sectional coverage with a long observation period and changing tradable universes. Each sample uses 126 historical trading sessions to predict the subsequent 21-session log return. The prepared training,
validation, and test splits begin in 2014, 2022,
and 2024, respectively, with testing extending
through June 2026. Samples whose target windows cross split boundaries
are excluded. Global normalization statistics are
estimated from the training split only, whereas
Stage~1 input scaling uses the historical fitting
pool associated with each prediction block.
Table~\ref{tab:dataset_statistics} reports prepared
sample counts before model-specific filtering.
Stock-JEPA representation and readout training begin
in 2016, when out-of-sample priors become available;
2014--2015 provides the initial Stage~1 fitting history.

\begin{table}[htbp]
\centering
\small
\setlength{\tabcolsep}{5pt}
\caption{Statistics of the full datasets after sample construction and temporal boundary purging}
\label{tab:dataset_statistics}
\begin{tabular}{lrrrrrl}
\toprule
Dataset
& Securities
& Valid samples
& Train
& Validation
& Test
& Test forecast dates \\
\midrule
\textsc{CN\_ALL}
& 5,669
& 8,839,281
& 3,883,006
& 1,990,842
& 2,965,433
& 2024-01-02--2026-06-26 \\
\textsc{US\_ALL}
& 10,152
& 15,882,433
& 7,674,081
& 3,169,353
& 5,038,999
& 2024-01-02--2026-06-24 \\
\bottomrule
\end{tabular}
\end{table}

\subsection{Expanding-window prior estimation.}
\label{app:temporal_protocol}
Stage~1 follows the expanding-window estimation principle of \citet{gu2020empirical} (Figure~\ref{fig:temporal_protocol}).
Starting from 2014--2015, it generates annual out-of-sample priors for 2016--2021. Each fit uses only historical samples whose complete outcome windows end before the corresponding prediction period, with input scaling and residual scales estimated from the same pool. For validation and testing, Stage~1 is fitted using eligible history available before each period and held fixed within that period. Stages~2 and~3 train on eligible 2016--2021 samples paired with their annual out-of-sample priors; checkpoints are selected on validation data and fixed during testing.
\begin{figure}[t]
    \centering
    \includegraphics[width=0.8\linewidth]{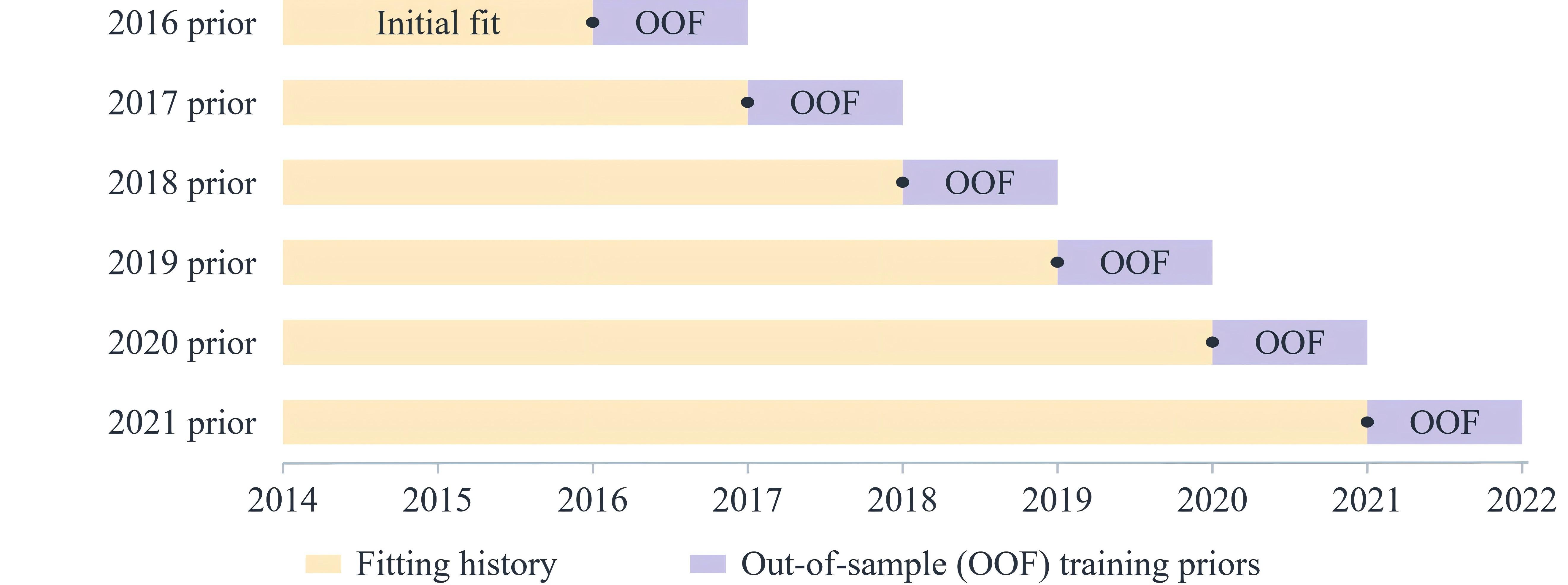}
    \caption{Schematic of expanding-window prior estimation. Yellow bars indicate fitting histories, and black dots mark fitting or refitting points.}
    \label{fig:temporal_protocol}
\end{figure}

\section{Cross-Sectional and Portfolio Evaluation}
\label{app:portfolio_details}

\paragraph{Information coefficients.}
For prediction date $t$, let $\mathcal U_t$ denote the evaluated
asset cross-section, $\widehat y_{i,t}$ the prediction, and
$y_{i,t}$ the realized 21-day log return.
We compute
\begin{equation}
\begin{aligned}
\mathrm{IC}_t
&=
\operatorname{Corr}_{i\in\mathcal U_t}
(\widehat y_{i,t},y_{i,t}),\\
\mathrm{RankIC}_t
&=
\operatorname{Corr}_{i\in\mathcal U_t}
\left(
\operatorname{rank}(\widehat y_{i,t}),
\operatorname{rank}(y_{i,t})
\right).
\end{aligned}
\end{equation}
Ranks use averages for ties.
Dates with fewer than three valid prediction--target pairs
or constant predictions or targets are excluded.
Reported IC and RankIC are equally weighted means
across valid dates and are multiplied by 100 in the tables.

\paragraph{Portfolio formation.}
We use a fixed selection size of $K=30$, consistent with
the Top-30 evaluation design used by \citet{li2024master},
while specifying our own holding and accounting protocol.
On each prediction date $t$, we select the $K$
highest-ranked assets for CN\_ALL and the $K$ highest-
and lowest-ranked assets for US\_ALL from the eligible
evaluation universe.
At least $K$ and $2K$ eligible assets are required,
respectively.
Ties are resolved by asset order in the prediction file.

Denoting the selected sets by $\mathcal L_t$ and
$\mathcal S_t$, each new batch uses weights
\begin{equation}
w^{\mathrm{CN}}_{i,t}
=
\frac{\mathbf{1}\{i\in\mathcal L_t\}}{K},
\qquad
w^{\mathrm{US}}_{i,t}
=
\frac{
\mathbf{1}\{i\in\mathcal L_t\}
-
\mathbf{1}\{i\in\mathcal S_t\}
}{K}.
\end{equation}
CN\_ALL therefore uses equal-weighted long-only batches,
whereas US\_ALL uses equal long and short notionals.
The number of selected stocks applies to each batch,
not to the aggregate account across overlapping batches.

Predictions based on information through the close of
session $t$ are executed at the next open, $t+1$.
Each batch maintains fixed position units until the close
of $t+H$, where $H=21$.
Execution and daily valuation use the datasets'
consistently adjusted open and close prices.
Realized portfolio returns are neither clipped nor winsorized.

\paragraph{Overlapping batches and capital allocation.}
All overlapping batches share one cash and net asset
value (NAV) account, initialized with unit cash.
Let $V_t$ and $C_t$ denote closing NAV and cash before
the next session's entry.
With proportional transaction-cost rate $c=0.0005$,
the new batch's side budget is
\begin{equation}
B_t^{\mathrm{CN}}
=
\min\left\{
\frac{V_t}{H},
\frac{\max(C_t,0)}{1+c}
\right\},
\qquad
B_t^{\mathrm{US}}
=
\frac{V_t}{H}.
\end{equation}
The CN budget covers purchases and entry costs without
cash borrowing.
For US\_ALL, $B_t$ is allocated to each leg, giving
initial gross notional $2B_t$ for the new batch.
Signed position units are
\begin{equation}
q_i^{(t)}
=
\frac{B_t w_{i,t}}{P^{\mathrm{open}}_{i,t+1}}.
\end{equation}
These are fractional units in the adjusted-price coordinate.
Existing batches are not rebalanced to restore their
initial weights as prices change.

\paragraph{Transaction costs and turnover.}
We charge 5 basis points per unit of executed notional
on both purchases and sales, including short sales
and short covering.
This rate is consistent with the low-cost specification
examined by \citet{conlon2025factorcovariance};
here, costs are applied directly to executed trades.

Let $\mathcal E_d$ contain all batch-level executions
on trading day $d$.
For execution $e$, let $\Delta q_e$ be the signed change
in position units and $P_e$ its execution price.
Gross traded notional and transaction costs are
\begin{equation}
A_d
=
\sum_{e\in\mathcal E_d}|\Delta q_e|P_e,
\qquad
F_d=cA_d,
\qquad c=0.0005.
\end{equation}
Entry trades use opening prices, and scheduled exits
use closing prices.
Costs are charged at both entry and exit, including
the final liquidation.
Opposing trades from different batches are not netted
before calculating fees.

Daily turnover is reported as
\begin{equation}
\mathrm{TO}_d
=
\frac{A_d}{2V_{d-1}},
\qquad
\overline{\mathrm{TO}}
=
\frac{1}{D}\sum_{d=1}^{D}\mathrm{TO}_d.
\end{equation}
Thus, transaction costs relative to preceding NAV equal
$F_d/V_{d-1}=2c\,\mathrm{TO}_d$.
This turnover measure uses executed notional rather than
differences between successive target-weight vectors.

\paragraph{Daily accounting.}
Cash reflects every execution and its associated fee,
including short-sale proceeds and short-covering payments:
\begin{equation}
C_d
=
C_{d-1}
-
\sum_{e\in\mathcal E_d}\Delta q_e P_e
-
F_d.
\end{equation}
With signed aggregate holdings $Q_{i,d}$ after the day's
executions, closing NAV and daily net returns are
\begin{equation}
V_d
=
C_d+\sum_i Q_{i,d}P^{\mathrm{close}}_{i,d},
\qquad
R_d^{\mathrm{net}}
=
\frac{V_d}{V_{d-1}}-1.
\end{equation}
Aggregate exposure evolves with prices and NAV.
Evaluation includes the initial position build-up
and continues through the final batch's scheduled exit.

The backtest assumes execution at the scheduled quoted
prices and does not model market impact, execution
restrictions, stock-borrow availability, borrowing fees,
or funding interest.
It imposes no early stop-loss or broker margin liquidation.
If aggregate NAV becomes non-positive at an observed
opening, closing, or post-trade valuation, the run is
terminated and reported as an account failure.

\paragraph{Performance aggregation.}
Using $D$ daily returns and 252 trading days per year,
arithmetic annualized return and Sharpe ratio are
\begin{equation}
\mathrm{Net\,AR}
=
252\,\overline R^{\mathrm{net}},
\qquad
\mathrm{Net\,SR}
=
\sqrt{252}\,
\frac{\overline R^{\mathrm{net}}}
{\operatorname{sd}(R^{\mathrm{net}})},
\end{equation}
where $\operatorname{sd}$ is the sample standard deviation
and the risk-free rate is zero.
Maximum drawdown uses the same daily closing NAV series,
including initial wealth $V_0=1$:
\begin{equation}
\mathrm{MDD}
=
\min_{0\leq d\leq D}
\left(
\frac{V_d}{\max_{0\leq u\leq d}V_u}-1
\right).
\end{equation}
Net AR and MDD are multiplied by 100 in the tables.

For a prescribed set of $S$ runs that all complete,
each metric is summarized by its arithmetic mean
and sample standard deviation across runs.
If any prescribed run fails, the failure is reported
rather than silently excluded from the aggregate.

\section{Component Ablation Details}
\label{app:component_ablation}

\paragraph{Pipeline configurations.}
S1-only uses the 16-dimensional prior vector $\mathbf{P}_{i,t}$,
comprising predictive means and log-scales, with scales clipped
below at $10^{-4}$ before taking logarithms.
S3-only uses the 40-dimensional historical summary
$\mathbf{s}_{i,t}$ obtained by mean-pooling stock and market
descriptors.
S2-only supplies a 128-dimensional prediction of the complete
future embedding to a linear readout, whereas S1+S2 supplies
$[\mathbf{P}_{i,t};\widehat{\mathbf{Z}}_{i,t}^{\mathrm{anc}};
\widehat{\boldsymbol{\Delta}}_{i,t}]\in\mathbb{R}^{272}$.
Both use training-set standardization and fit linear readouts
to future 21-day returns with representations frozen.
The full model adds $\mathbf{s}_{i,t}$ to the S1+S2 features
and replaces the linear readout with an MLP.

\paragraph{Representation training.}
S2-only is trained independently from random initialization
without a prior branch, prior initialization, financial semantic
heads, or prior loss. It predicts targets from an EMA target
encoder and retains diagonal target whitening and
variance-covariance regularization.
S1+S2 uses the prior-anchored representation checkpoint for the
corresponding seed, whose predictor estimates future latent
revisions relative to a financial anchor.

\paragraph{Results and interpretation.}
Table~\ref{tab:component_ablation_full} shows that S1+S2
outperforms every standalone configuration across all four
mean metrics in both markets.
These gains with a linear readout support the predictive value
of the prior-anchored features.
Adding the historical summary and an MLP further raises RankIC
by 1.28 and 1.07 percentage points and net Sharpe by 0.40 and
0.24 on CN\_ALL and US\_ALL, respectively.
The full model achieves the highest means on all four metrics.
Matched input ablations are reported separately in
Appendix~\ref{app:matched_mlp_readout}.

\begin{table}[htbp]
\centering
    \caption{Pipeline component ablations
    (mean$\pm$SD over three seeds).
    Bold marks the highest mean within each market.
    IC and RankIC are correlations multiplied by 100;
    Net AR is arithmetic annualized return (\%).
    S3-only denotes the history-only baseline.}
    \label{tab:component_ablation_full}
    \small
    \setlength{\tabcolsep}{5.2pt}
    \renewcommand{\arraystretch}{1.10}
    \begin{tabular}{lcccc}
    \toprule
    Configuration
    & IC (\%) $\uparrow$
    & RankIC (\%) $\uparrow$
    & Net AR (\%) $\uparrow$
    & Net Sharpe $\uparrow$ \\
    \midrule
    \multicolumn{5}{l}{\textit{CN\_ALL}} \\
    S1-only
    & $4.18\pm3.10$ & $5.15\pm3.92$
    & $14.03\pm6.41$ & $0.55\pm0.62$ \\
    S2-only
    & $4.40\pm2.95$ & $6.21\pm3.17$
    & $21.43\pm5.96$ & $0.85\pm0.27$ \\
    S3-only
    & $4.21\pm0.58$ & $6.76\pm0.43$
    & $21.04\pm3.84$ & $0.89\pm0.14$ \\
    S1+S2
    & $8.86\pm1.28$ & $9.94\pm1.23$
    & $33.86\pm11.28$ & $1.10\pm0.27$ \\
    \textbf{S1+S2+S3}
    & $\boldsymbol{10.70\pm0.61}$
    & $\boldsymbol{11.22\pm0.24}$
    & $\boldsymbol{45.66\pm5.12}$
    & $\boldsymbol{1.50\pm0.28}$ \\
    \midrule
    \multicolumn{5}{l}{\textit{US\_ALL}} \\
    S1-only
    & $7.70\pm3.83$ & $4.62\pm3.20$
    & $23.58\pm6.30$ & $0.63\pm0.30$ \\
    S2-only
    & $6.91\pm4.71$ & $7.92\pm2.80$
    & $39.25\pm24.36$ & $0.81\pm0.10$ \\
    S3-only
    & $6.94\pm1.63$ & $8.72\pm1.10$
    & $35.41\pm1.32$ & $0.73\pm0.10$ \\
    S1+S2
    & $10.40\pm1.27$ & $11.89\pm0.66$
    & $45.38\pm13.30$ & $1.00\pm0.19$ \\
    \textbf{S1+S2+S3}
    & $\boldsymbol{11.67\pm0.79}$
    & $\boldsymbol{12.96\pm0.48}$
    & $\boldsymbol{68.20\pm15.92}$
    & $\boldsymbol{1.24\pm0.20}$ \\
    \bottomrule
    \end{tabular}
\end{table}

\section{Matched MLP Readout Ablation}
\label{app:matched_mlp_readout}

\paragraph{Setup.}
We test the incremental utility of revisions under a matched
non-linear readout. For each market and seed, all configurations
reuse the same frozen checkpoint with the following inputs:
\begin{align*}
\mathrm{R0}:&\quad [\mathbf{s};\mathbf{P}],
& \mathrm{R1}:&\quad [\mathbf{s};\mathbf{P};\mathbf{z}^{a}],\\
\mathrm{R2}:&\quad
[\mathbf{s};\mathbf{P};\mathbf{z}^{a};
\widehat{\boldsymbol{\Delta}}],
& \mathrm{R3}:&\quad
[\mathbf{s};\mathbf{P};\mathbf{z}^{a};\mathbf{h}].
\end{align*}
Here, $\mathbf{s}$ and $\mathbf{P}$ are the 40-dimensional
state and 16-dimensional prior vectors; the anchor
$\mathbf{z}^{a}$, revision $\widehat{\boldsymbol{\Delta}}$,
and context representation $\mathbf{h}$ are each 128-dimensional. 
\paragraph{Matched training and evaluation.}
All configurations use a $312\!\rightarrow\!256\!\rightarrow\!1$
MLP with GELU and dropout of 0.05.
Inputs are standardized using training-set statistics, with
unused coordinates set to zero and no concatenated LayerNorm.
Within each market and seed, MLP initialization, training samples,
and date ordering are identical.
Only the MLP is trained; representation parameters and
normalization buffers remain frozen.
We use AdamW and SmoothL1 loss ($\beta=1$) on normalized
21-day returns, without a ranking loss, for six epochs with
at most 1,000 date-group updates per epoch.
The checkpoint with the lowest loss on the designated
validation-selection subset is retained.
Data splits, target processing, and portfolio evaluation follow
the main experiments, including 5-bps transaction costs per
unit of executed notional at entry and exit.
We reuse the main experiments' three seeds and frozen checkpoints.
This ablation replaces the main readout's LayerNorm with
training-set feature standardization and retrains only the MLP.
R2 therefore differs from the main evaluation in both input
normalization and realized training randomness.

\begin{table}[!htbp]
    \centering
    \caption{Matched MLP readout ablations on frozen representations.
    Entries are mean$\pm$SD over three seeds.
    Best means within each market are bold.
    IC and RankIC are correlations multiplied by 100;
    Net AR is arithmetic annualized return (\%).}
    \label{tab:matched_mlp_readout}
    \small
    \setlength{\tabcolsep}{5.2pt}
    \renewcommand{\arraystretch}{1.05}
    \begin{tabular}{lcccc}
        \toprule
        Configuration
        & IC (\%) $\uparrow$
        & RankIC (\%) $\uparrow$
        & Net AR (\%) $\uparrow$
        & Net Sharpe $\uparrow$ \\
        \midrule
        \multicolumn{5}{l}{\textit{CN\_ALL}} \\
        R0: State + Prior
        & $8.46\pm0.61$
        & $9.72\pm0.60$
        & $39.46\pm7.56$
        & $1.47\pm0.30$ \\
        R1: + Anchor
        & $9.68\pm0.97$
        & $9.24\pm2.12$
        & $48.83\pm11.29$
        & $1.46\pm0.32$ \\
        \textbf{R2: + Anchor + Revision}
        & $\mathbf{11.04\pm0.39}$
        & $\mathbf{11.26\pm0.63}$
        & $\mathbf{59.46\pm4.17}$
        & $\mathbf{1.87\pm0.23}$ \\
        R3: + Anchor + Context
        & $10.55\pm0.67$
        & $10.13\pm1.04$
        & $56.65\pm4.49$
        & $1.79\pm0.02$ \\
        \midrule
        \multicolumn{5}{l}{\textit{US\_ALL}} \\
        R0: State + Prior
        & $7.62\pm0.27$
        & $8.28\pm0.60$
        & $21.74\pm9.75$
        & $0.44\pm0.39$ \\
        R1: + Anchor
        & $10.28\pm0.06$
        & $9.85\pm0.22$
        & $50.63\pm11.26$
        & $0.78\pm0.30$ \\
        \textbf{R2: + Anchor + Revision}
        & $\mathbf{12.11\pm0.55}$
        & $\mathbf{13.52\pm0.58}$
        & $\mathbf{74.33\pm10.58}$
        & $\mathbf{1.31\pm0.21}$ \\
        R3: + Anchor + Context
        & $10.68\pm0.91$
        & $11.25\pm0.61$
        & $70.37\pm3.08$
        & $1.01\pm0.12$ \\
        \bottomrule
    \end{tabular}
\end{table}
\paragraph{Results.}
Table~\ref{tab:matched_mlp_readout} shows that R2 achieves
the highest means across all four metrics in both markets.
Adding the anchor (R1 versus R0) improves all four metrics
on US\_ALL. On CN\_ALL, IC and net AR improve, but RankIC
falls from 9.72\% to 9.24\% and net Sharpe is nearly
unchanged (1.47 versus 1.46).
Adding the revision (R2 versus R1) improves all four metrics
in both markets, raising RankIC by 2.02 and 3.67 percentage
points and net Sharpe by 0.41 and 0.53 on CN\_ALL and
US\_ALL, respectively.
With the state, prior, and anchor retained, these gains
support the revision's incremental predictive and portfolio value.

R2 versus R3 compares predicted revisions with direct context
inputs under identical input dimensionality and MLP capacity.
R2 achieves higher means on all four metrics in both markets,
including RankIC gains of 1.13 and 2.27 percentage points
and net Sharpe gains of 0.08 and 0.30 on CN\_ALL and
US\_ALL, respectively.
Thus, the benefit extends beyond simply supplying the
historical context representation to the readout.
These results suggest that expressing historical information
as a predicted revision relative to the financial anchor
makes it more useful for downstream ranking and portfolio
construction.
\section{Controlled Ablation of Representation-Learning Mechanisms}
\label{app:controlled_representation_ablation}

\paragraph{Configurations.}
Separate-EMA implements our representation-learning mechanism under the controlled ablation protocol: the online prior projector receives separate anchor supervision, while revision prediction uses an EMA anchor without propagating gradients to the projector. Separate-online retains this objective separation and gradient routing but uses the online anchor during training and inference. Joint-online instead jointly optimizes total-representation prediction, removing separate anchor supervision and gradient blocking. Ordinary-JEPA directly predicts the future representation, with prior conditioning zeroed and the unused projector frozen. The separate objectives share one batch and backward pass.

\paragraph{Controlled training.}
For each market and seed, all mechanisms continue from the same S2 checkpoint with matched initialization, data splits, masking, and budget: four epochs of 1,000 updates with batch size 256. We use AdamW (learning rate $10^{-4}$, weight decay 0.01), gradient clipping at 1.0, target-encoder EMA decay 0.996, and anchor EMA decay 0.99 where applicable. Additional variance, covariance, and orthogonality regularizers are disabled; the epoch-four checkpoint is retained. All mechanisms use unclipped diagonal EMA whitening with detached, mechanism-specific statistics. For $\widehat z=a+Du+\mu$, the joint loss is
\begin{equation}
    \mathcal{L}_{\mathrm{joint}}
    =\frac{1}{d}\left\|D^{-1}(\widehat z-y)\right\|_2^2,
\end{equation}
where $y$ is the detached future target, $a$ the anchor,
$u$ the predictor output, and $D,\mu$ the whitening scale and mean.
This protocol evaluates fixed-budget continuation.

\paragraph{Matched readout.}
All representation parameters and buffers are frozen. Every mechanism uses the same input
$[\mathbf{s}_{40};\mathbf{P}_{16};\widehat z_{128}]
\in\mathbb{R}^{184}$
and a $184\!\rightarrow\!256\!\rightarrow\!1$ MLP, with matched initialization and training-date ordering. Here, $\widehat z$ is the anchor-plus-revision prediction or the direct Ordinary-JEPA prediction. Other readout settings and validation selection follow Appendix~\ref{app:matched_mlp_readout}; test splits and portfolio evaluation follow the main experiments.

\paragraph{Results and interpretation.}
Table~\ref{tab:controlled_representation_ablation} shows that
Separate-EMA achieves the highest means on three of four
metrics in each market, including both portfolio metrics.
Relative to Joint-online, it raises net AR by 11.18 and
6.53 percentage points and net Sharpe by 0.26 and 0.18
on CN\_ALL and US\_ALL, respectively.
Joint-online still leads CN\_ALL RankIC and US\_ALL IC.
Separate-online also achieves higher mean net AR and
Sharpe than Joint-online in both markets, despite lower RankIC.
Holding separate supervision and gradient routing fixed,
the EMA anchor yields small changes on CN\_ALL and improves
all four US\_ALL means, including net AR by 5.24 percentage
points and net Sharpe by 0.14.

Separate-EMA also outperforms Ordinary-JEPA on all four mean
metrics in both markets, with RankIC gains of 1.94 and
3.40 percentage points on CN\_ALL and US\_ALL, respectively.
All readouts receive the same state and prior inputs and
equally sized predicted representations.
These gains support the downstream value of prior-anchored
representation learning beyond supplying the prior to the
readout.
Together, the comparisons favor separate supervision with
an EMA anchor for downstream portfolio performance under
the matched continuation protocol.
\begin{table}[!htbp]
    \centering
    \caption{Controlled ablation of representation-learning mechanisms.
    Entries are mean$\pm$SD across seeds;
    best means within each market are bold.
    IC and RankIC are correlations multiplied by 100;
    Net AR is arithmetic annualized return (\%).}
    \label{tab:controlled_representation_ablation}
    \small
    \setlength{\tabcolsep}{5.2pt}
    \renewcommand{\arraystretch}{1.05}

    \begin{tabular}{lcccc}
        \toprule
        Training mechanism
        & IC (\%) $\uparrow$
        & RankIC (\%) $\uparrow$
        & Net AR (\%) $\uparrow$
        & Net Sharpe $\uparrow$ \\
        \midrule

        \multicolumn{5}{l}{\textit{CN\_ALL}} \\
        \textbf{Separate-EMA}
        & $\mathbf{10.90\pm0.33}$
        & $11.11\pm0.94$
        & $\mathbf{55.42\pm5.46}$
        & $\mathbf{1.51\pm0.26}$ \\
        Separate-online
        & $10.85\pm0.30$
        & $11.13\pm0.68$
        & $53.92\pm3.66$
        & $1.48\pm0.26$ \\
        Joint-online
        & $10.43\pm0.53$
        & $\mathbf{11.49\pm0.60}$
        & $44.24\pm2.35$
        & $1.25\pm0.12$ \\
        Ordinary-JEPA
        & $8.73\pm0.81$
        & $9.17\pm0.59$
        & $40.77\pm8.65$
        & $1.01\pm0.23$ \\

        \midrule
        \multicolumn{5}{l}{\textit{US\_ALL}} \\
        \textbf{Separate-EMA}
        & $12.27\pm0.34$
        & $\mathbf{12.45\pm0.74}$
        & $\mathbf{67.13\pm13.68}$
        & $\mathbf{1.27\pm0.13}$ \\
        Separate-online
        & $11.58\pm0.32$
        & $12.14\pm0.75$
        & $61.89\pm10.87$
        & $1.13\pm0.10$ \\
        Joint-online
        & $\mathbf{12.47\pm0.07}$
        & $12.37\pm1.20$
        & $60.60\pm7.46$
        & $1.09\pm0.20$ \\
        Ordinary-JEPA
        & $7.64\pm0.24$
        & $9.05\pm0.90$
        & $32.39\pm6.59$
        & $0.50\pm0.09$ \\

        \bottomrule
    \end{tabular}
\end{table}

\section{Signal Recovery in a Controlled Dynamic-Factor Environment}
\label{app:simulation_recovery}

\paragraph{Data-generating process.}
We examine which known financial signal components are recoverable
from learned representations.
Each synthetic world contains 256 assets and two independent,
unit-variance stationary states:
\begin{equation}
s_{t+1}=\rho_s s_t+\sqrt{1-\rho_s^2}\epsilon^s_{t+1},
\qquad
f_{t+1}=\rho_f f_t+\sqrt{1-\rho_f^2}\epsilon^f_{t+1},
\end{equation}
where $\rho_s=0.99$, $\rho_f=0.80$, and the state innovations
are independent standard Gaussian variables.
Daily log returns follow
\begin{equation}
\ell_{i,t+1}
=
a_s\beta_i s_t
+\gamma a_f(\lambda_i+\kappa c_i s_t)f_t
+\sigma_m b_i\xi_{t+1}
+\sigma_e\varepsilon_{i,t+1},
\end{equation}
with $a_s=0.0005$, $a_f=0.001$, $\kappa=0.5$,
$\sigma_m=0.008$, and $\sigma_e=0.015$.
Fixed exposures $\beta_i,\lambda_i,c_i$ are independently
drawn from standard Gaussian distributions and cross-sectionally
standardized; $b_i\sim\mathcal U(0.5,1.5)$.
The common and idiosyncratic noise terms are independent
of the current information set.

Models observe asset exposures, noisy proxies
$q^s_t=s_t+0.5\eta^s_t$ and
$q^f_{i,t}=\lambda_i f_t+0.5\eta^f_{i,t}$,
and causal historical summaries, forming 28 stock-level
and 12 market-level channels.
The true states and diagnostic targets are never used as
inputs or supervision for representation learning.

\paragraph{Diagnostic targets.}
For $H=21$, let $C_H(\rho)=\sum_{k=0}^{H-1}\rho^k$.
Conditional on the true current states, expected cumulative
log return decomposes into
\begin{align}
m^{\mathrm{slow}}_{i,t}
&=a_s\beta_i C_H(\rho_s)s_t,\\
m^{\mathrm{fast}}_{i,t}
&=\gamma a_f\left[
\lambda_i C_H(\rho_f)f_t
+\kappa c_i C_H(\rho_s\rho_f)s_tf_t
\right].
\end{align}
The fast component includes both the direct fast-factor
contribution and its interaction with the slow state.
We also probe the independent future return noise
\begin{equation}
u_{i,t}
=
\sum_{k=1}^{H}
\left(\sigma_m b_i\xi_{t+k}
+\sigma_e\varepsilon_{i,t+k}\right)
\end{equation}
as a negative control.
This noise excludes unpredictable changes in the latent
states and therefore is not the complete return forecast residual.

\paragraph{Representation training.}
We use $\gamma=1$ and 2,016 representation-training dates.
A frozen Ridge prior is fitted on an earlier, separate
4,000-date calibration block to predict return, volatility,
log drawdown, and market-relative return at horizons of
5 and 21 days.
Contexts contain 126 sessions, and JEPA targets use
21-session future windows.
Validation and test sets contain 504 and 1,008 prediction
dates, respectively, with 21-session embargoes between
chronological partitions.

We compare Separate-EMA, Separate-online, Joint-online,
Ordinary-JEPA, and Direct supervised.
Separate variants use distinct anchor and revision objectives,
with the anchor detached from the revision loss;
Separate-EMA uses an EMA anchor and Separate-online
uses the online anchor.
Joint-online uses a single summed-latent objective without
detaching the anchor.
Ordinary-JEPA predicts future representations without
the prior-relative decomposition, while Direct supervised
predicts returns directly.

All models are initialized from scratch with a 128-dimensional,
four-layer, eight-head context Transformer and trained for
2,000 updates using AdamW
(learning rate $10^{-4}$, weight decay 0.01),
batch size 128, and gradient clipping at 1.0.
The target-encoder EMA decay is 0.996;
the prior EMA decay in Separate-EMA is 0.99.
JEPA checkpoints use the final update, while Direct supervised
selects its checkpoint by validation return MSE.
Diagnostic targets are not used for checkpoint selection.

\paragraph{Frozen linear probes.}
After training, all representation parameters and statistical
buffers are frozen.
For decomposed models, we probe the anchor, revision,
and their concatenation:
$Z^{\mathrm{anc}}$, $\widehat{\Delta}$, and
$[Z^{\mathrm{anc}};\widehat{\Delta}]$,
with dimensions 128, 128, and 256.
Ordinary-JEPA uses its predicted future representation;
Direct supervised uses its context representation.
Probes receive no additional raw features or prior statistics.

Each representation--target pair uses a Ridge probe with
an unpenalized intercept.
Probes are fitted on a separate trajectory from the same world,
using 252 training dates and 126 validation dates separated
by a 21-session embargo.
This trajectory is excluded from representation training.
Features are standardized using probe-training statistics,
and the regularization coefficient is selected by validation
MSE from $\{10^{-4},10^{-2},1,10^2,10^4\}$.
Evaluation uses the original 1,008-date test set.

\paragraph{Recovery metric and aggregation.}
For target $q$ and probe prediction $\widehat q$, recovery
is the held-out coefficient of determination expressed
as a percentage:
\begin{equation}
\mathrm{Recovery}(q)
=
100\left(
1-
\frac{\sum_{(i,t)\in\mathcal T}
(q_{i,t}-\widehat q_{i,t})^2}
{\sum_{(i,t)\in\mathcal T}
(q_{i,t}-\overline q_{\mathcal T})^2}
\right).
\end{equation}
The metric pools all test asset--date observations.
A value of $100\%$ indicates perfect recovery;
$0\%$ matches the test-mean constant reference.
Negative values are retained and indicate worse performance
than this reference.

We use five independent world seeds, each with three
optimization seeds.
Scores are first averaged within each world, then summarized
by the mean and sample standard deviation across worlds.
All runs are retained; reported deviations describe
cross-world variation, not confidence intervals.
Recovery measures linear accessibility, and branch-wise
scores are not additive.
Concatenation increases probe input capacity.
The generative slow--fast components also differ from
the theoretical prior-conditional anchor--revision decomposition.

\begin{table}[!htbp]
    \centering
    \caption{Frozen linear-probe recovery in the controlled
    dynamic-factor environment.
    Entries are held-out $R^2$ (\%), averaged over three
    optimization seeds within each world and reported as
    mean$\pm$sample standard deviation across five worlds.
    Best means for predictable components are bold;
    independent future return noise serves as a negative control.}
    \label{tab:simulation_recovery}
    \small
    \setlength{\tabcolsep}{4pt}
    \renewcommand{\arraystretch}{1.05}
    \begin{tabular}{llccc}
        \toprule
        Training mechanism
        & Probe input
        & Slow component
        & Fast component
        & Noise \\
        \midrule
        Separate-EMA
        & Anchor
        & $\mathbf{66.14\pm10.49}$
        & $-2.49\pm2.26$
        & $-4.55\pm2.34$ \\
        & Revision
        & $-2.70\pm8.04$
        & $49.25\pm36.22$
        & $-16.95\pm14.51$ \\
        & Anchor + Revision
        & $62.22\pm17.52$
        & $\mathbf{55.24\pm37.93}$
        & $-29.36\pm25.87$ \\
        \midrule
        Separate-online
        & Anchor
        & $66.05\pm10.43$
        & $-2.56\pm2.34$
        & $-4.57\pm2.32$ \\
        & Revision
        & $-8.70\pm33.67$
        & $44.28\pm36.05$
        & $-15.76\pm12.92$ \\
        & Anchor + Revision
        & $60.96\pm16.97$
        & $45.76\pm38.14$
        & $-27.01\pm22.98$ \\
        \midrule
        Joint-online
        & Anchor
        & $62.67\pm8.73$
        & $-3.30\pm2.79$
        & $-4.69\pm2.85$ \\
        & Revision
        & $-20.59\pm19.03$
        & $43.02\pm44.27$
        & $-17.46\pm12.54$ \\
        & Anchor + Revision
        & $59.29\pm17.42$
        & $48.61\pm45.13$
        & $-26.96\pm22.08$ \\
        \midrule
        Ordinary-JEPA
        & Predicted future
        & $-42.58\pm120.61$
        & $45.84\pm42.68$
        & $-23.44\pm19.73$ \\
        Direct supervised
        & Context
        & $-37.61\pm58.70$
        & $41.56\pm25.79$
        & $-48.20\pm45.21$ \\
        \bottomrule
    \end{tabular}
\end{table}

\paragraph{Results and interpretation.}
Table~\ref{tab:simulation_recovery} shows complementary
signal recovery from the anchor and revision.
For Separate-EMA, the anchor achieves 66.14\% recovery of
the slow component but negative recovery of the fast component.
The revision exhibits the opposite pattern, recovering
49.25\% of the fast component while its slow-component
recovery remains negative.
Concatenation retains substantial recovery of both
components (62.22\% slow and 55.24\% fast), providing
joint access to persistent and faster predictable variation.
It does not, however, improve both scores over their
respective individual branches.

Ordinary-JEPA and Direct supervised also recover the fast
component (45.84\% and 41.56\%), but neither achieves positive
mean slow-component recovery under these linear probes.
This contrast suggests that explicit anchoring helps retain
linearly accessible persistent variation alongside faster
signals.
Similar branch specialization appears in Separate-online
and Joint-online.
Separate-EMA attains the highest mean slow recovery with
its anchor and fast recovery with concatenated features,
although substantial cross-world variation in fast recovery
limits conclusions about the ordering of mechanisms.

All configurations yield negative mean recovery of the
independent future return noise, consistent with its
unavailability at the forecast origin.
The probes therefore do not outperform the constant
reference on average for this negative control;
more negative scores do not imply better noise removal.
The clearest evidence is the complementary recovery of
known predictable components from the two branches.

\section{Revision Structure Diagnostics}
\label{app:revision_structure}

\paragraph{Cross-sectional structure.}
We analyse frozen 128-dimensional Stock-JEPA revisions
on US\_ALL without using future targets.
For the left panel, we remove daily cross-sectional means,
equally average daily covariance matrices, and convert
the resulting matrix into absolute correlations.
Over January 2024--November 2025, the displayed
representation has a mean off-diagonal absolute
correlation of 0.304 and an average daily effective
rank of 21.1.
Effective rank is $\exp(-\sum_j p_j\log p_j)$,
where $p_j$ denotes normalized singular values of
the centered daily representation matrix.

\paragraph{Temporal structure.}
The right panel displays October 2024--November 2025.
Revisions are standardized using training-set
coordinate-wise means and standard deviations,
then averaged across stocks on each prediction date.
We compute pairwise cosine distances
($1-\text{cosine similarity}$) between these daily
market-average vectors and average the distance
matrices over three runs.
Larger values indicate greater directional differences
between market-average revisions.
April 3, 2025 is marked as an event reference.

\section{Revision Response Diagnostic}
\label{app:revision_response}

We freeze the trained representation model and extract predicted
revision vectors using only information available at each
prediction date, without updating model parameters or whitening
statistics. For run $k$, let
$\widehat{\boldsymbol{\Delta}}^{(k)}_{i,t}\in\mathbb{R}^{128}$
denote the predicted revision. Its daily displacement is
\begin{equation}
D_t^{(k)}
=
\operatorname{median}_{i\in\mathcal{I}_t}
\left\|
\frac{
\widehat{\boldsymbol{\Delta}}^{(k)}_{i,t}
-
\widehat{\boldsymbol{\Delta}}^{(k)}_{i,t^-}
}{
\boldsymbol{\sigma}_{\mathrm{train}}^{(k)}
}
\right\|_2,
\end{equation}
where $t^-$ is the preceding available trading session and
$\mathcal{I}_t$ contains stocks observed on both dates.
Division is coordinate-wise; standard deviations are estimated
from 50,000 training samples and floored at $10^{-3}$.
The plotted curve is the unsmoothed mean
$D_t=\frac{1}{3}\sum_{k=1}^{3}D_t^{(k)}$.
Thus, scalar displacements, rather than latent coordinates,
are averaged across runs.

The dashed line, labelled ``99th percentile (full test),''
is the empirical 0.99 quantile of $D_t$ over all valid daily
differences in the corresponding market's complete test period:
January 2024--June 2026, comprising 598 differences for
CN\_ALL and 620 for US\_ALL.
It is computed separately for each market, including dates
outside the displayed interval.
Approximately 1\% of full-test daily displacements exceed this
reference level; it is neither a confidence interval nor a
cross-sectional percentile.
The threshold is used only for retrospective visualization,
not for training, model selection, or trading decisions.
Larger values indicate broader stock-level revision adjustments,
without specifying a bullish or bearish direction.
Selected event annotations illustrate temporal correspondence,
not causal effects or predictive anticipation.

\paragraph{Revision Responses across Markets.}
\label{app:revision_response_markets}
We examine revision responses in both CN\_ALL and US\_ALL
to assess whether the same pattern appears under different
market environments.
Major policy and macroeconomic events change the information
available to investors, creating a setting in which
context-dependent revisions should adjust the model's
prior-anchored representation.

Figure~\ref{fig:revision_response_both} shows pronounced
revision changes following the September 2024 policy package
and the April 2025 tariff escalation on CN\_ALL.
On US\_ALL, substantial changes accompany hawkish
monetary-policy repricing and renewed trade tensions.
Because the displacement statistic uses the cross-sectional
median of stock-level changes, these peaks reflect adjustments
across a broad set of assets rather than a few extreme
individual observations.
The corresponding responses across the two markets support
the role of the revision branch in incorporating changing
market information into the representation, complementing
the structured financial anchor.

\begin{figure}[!htbp]
    \centering
    \includegraphics[width=0.49\linewidth]{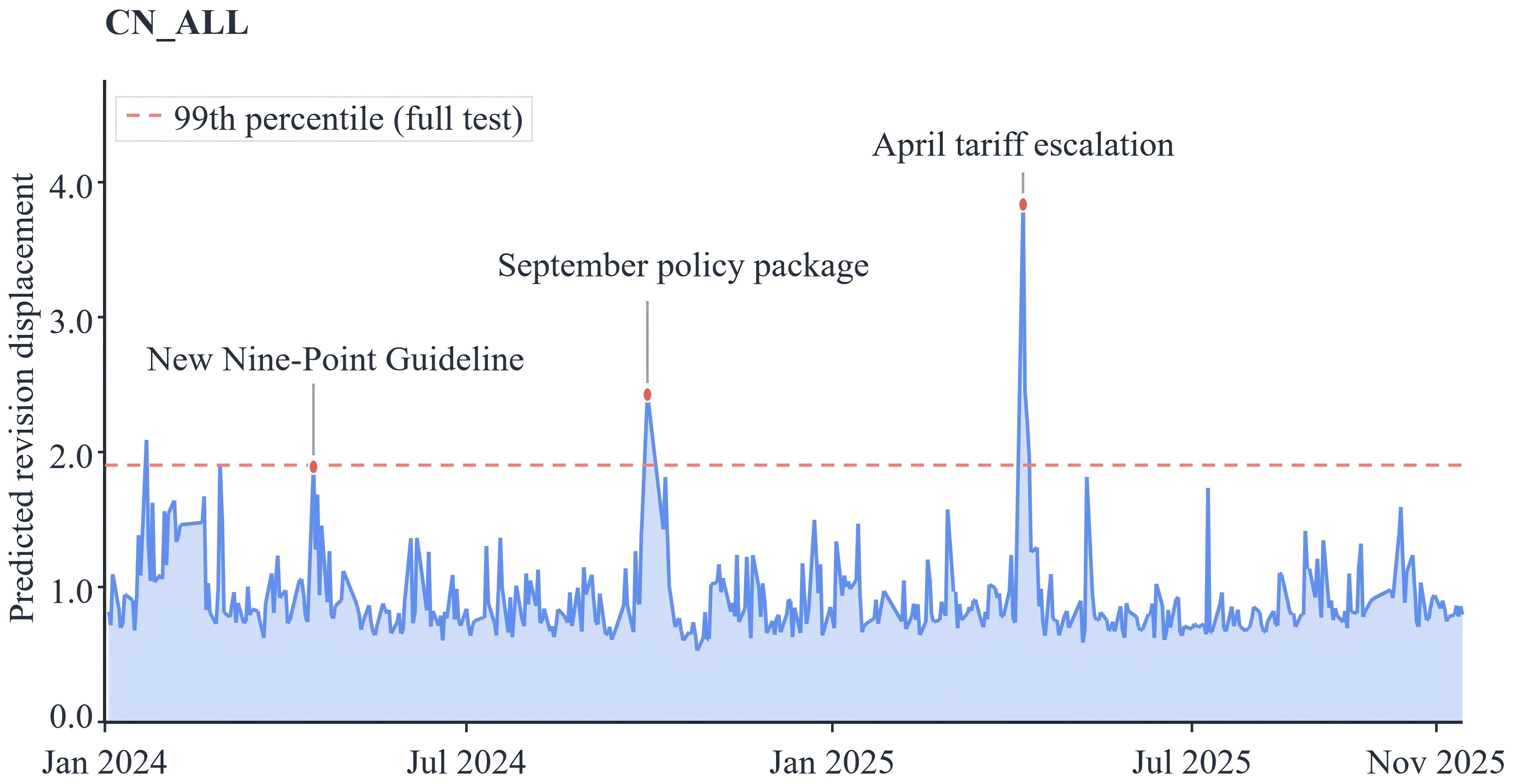}
    \hfill
    \includegraphics[width=0.49\linewidth]{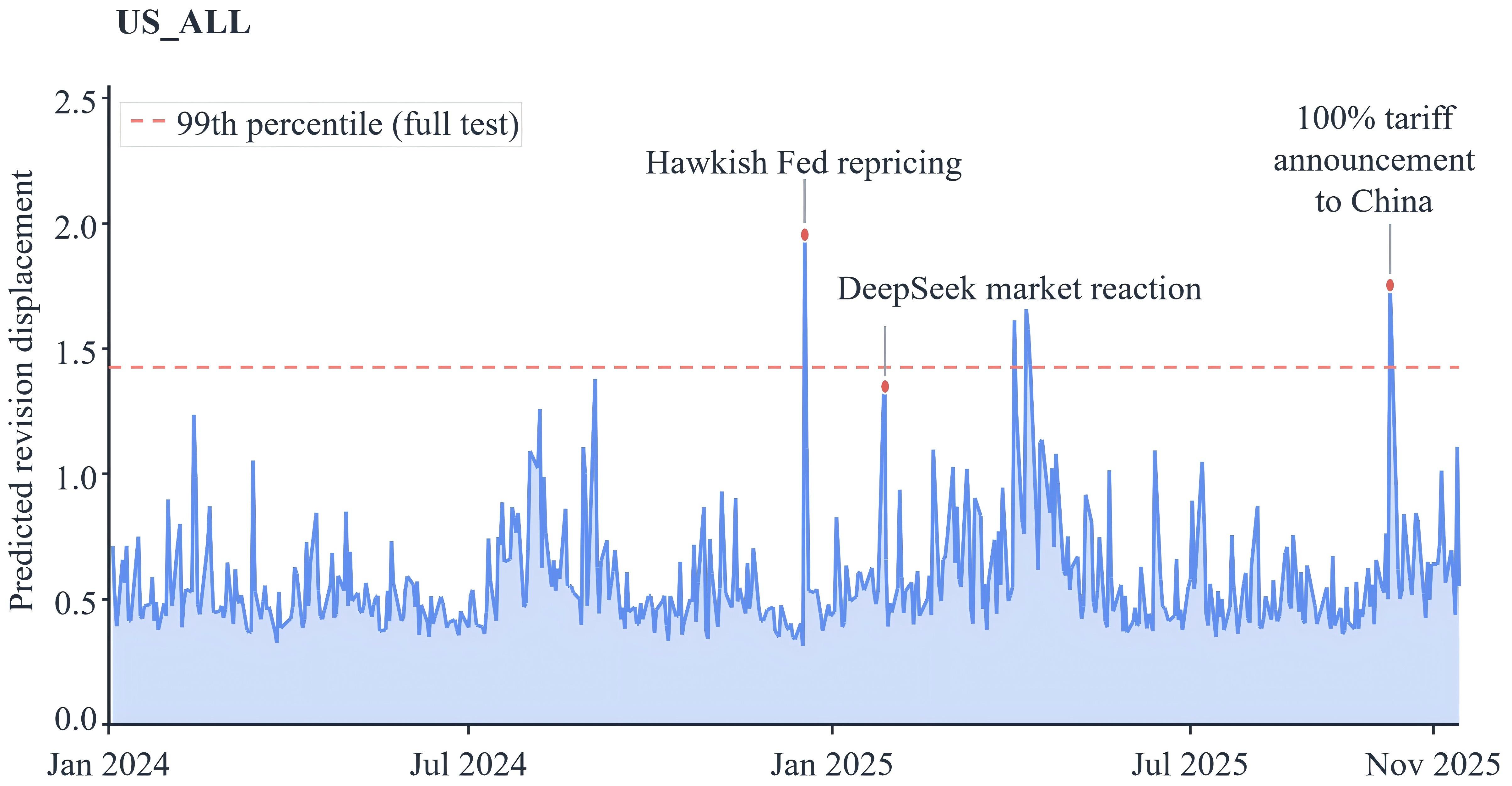}
    \caption{
    \textbf{Revision responses around major market events.}
    Daily revision displacement on CN\_ALL (left) and
    US\_ALL (right), averaged across three runs without smoothing.
    Peaks reflect broad stock-level adjustments; annotations
    identify selected market events.
    Dashed lines denote market-specific 99th percentiles
    computed over the complete test period, including dates
    outside the displayed interval.
    }
    \label{fig:revision_response_both}
\end{figure}

\begin{figure}[H]
    \centering
    \includegraphics[width=0.8\linewidth]{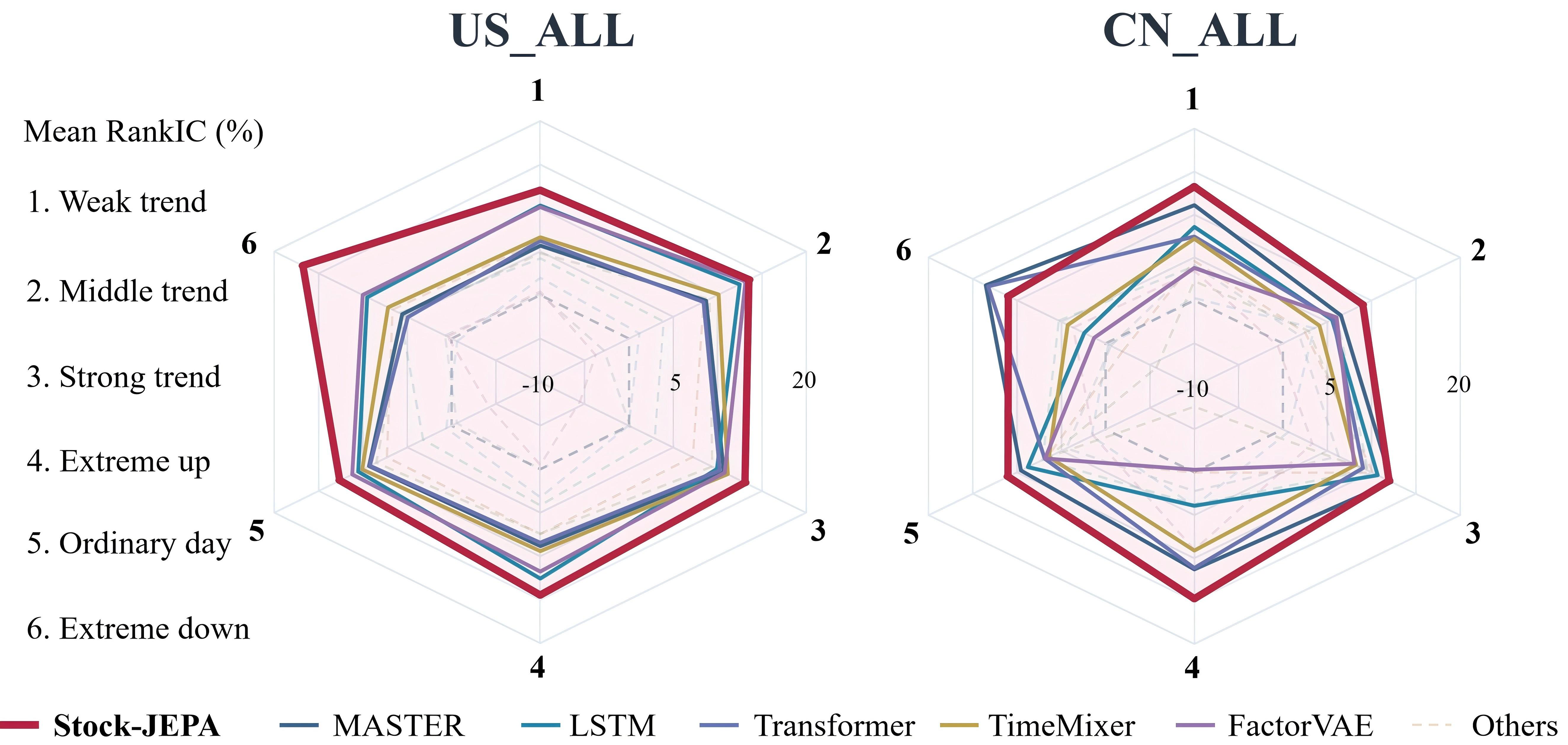}
    \caption{
    \textbf{Predictive utility across market states.}
    Mean RankIC (\%) over three runs on US\_ALL (left) and
    CN\_ALL (right), using a common radial scale.
    Trend states (axes 1--3) and daily shock states (axes 4--6)
    form separate partitions of prediction dates.
    Dashed lines denote the remaining baselines.
    }
    \label{fig:market_states}
\end{figure}

\section{Market-State Evaluation Details}
\label{app:market_states}

\paragraph{Evaluation setup.}
We reuse existing 21-day return predictions from \textsc{Stock-JEPA} and the 13 baselines, without state-specific retraining, readout adaptation, or checkpoint selection. CN\_ALL contains 599 prediction dates from January 2, 2024 to June 26, 2026; US\_ALL contains 621 dates from January 2, 2024 to June 24, 2026.

\paragraph{State definitions.}
We use the equal-weighted market return to define two separate classifications. Trend states use the trailing 63-day cumulative market return, with training-period 33rd and 67th percentiles defining weak, middle, and strong trends. Shock states use the current-day market return, with training-period 5th and 95th percentiles defining extreme down, ordinary day, and extreme up. Thresholds are estimated separately for each market and fixed during testing. Each date belongs to one trend group and one shock group; trend labels describe strength relative to the training distribution.

\paragraph{Metrics and interpretation.}
We compute daily cross-sectional Spearman correlations between predictions and subsequent 21-day cumulative log returns, then average equally across dates within each state and across three runs. Figure~\ref{fig:market_states} reports mean RankIC (\%) for both markets. \textsc{Stock-JEPA} achieves positive mean RankIC in every group, ranking first in all six US\_ALL groups and five of six CN\_ALL groups.
On extreme-down days in CN\_ALL, some baselines perform better. These exploratory results support predictive usefulness across the evaluated market conditions.

\end{document}